%% file: arXiv.tex
\documentclass{article}

\usepackage{hyperref}
\usepackage{url}
\usepackage{amsthm}
\usepackage{amsfonts}
\usepackage{amssymb}
\usepackage{graphicx}
\usepackage{amsmath} 
\usepackage{bm}
\usepackage{booktabs} 
\usepackage{multirow} 
\usepackage{caption}  
\usepackage[table]{xcolor}
\usepackage{booktabs}
\usepackage{mathtools}
\usepackage{amsthm}
\usepackage{etoolbox}
\usepackage{etoc}
\usepackage{soul}
\definecolor{myblue}{RGB}{174, 206, 255}
\definecolor{mypink}{RGB}{255, 197, 209} 

\definecolor{goodgreen}{RGB}{25, 120, 25}
\definecolor{badred}{RGB}{180, 25, 25}
\definecolor{boxframe}{RGB}{40, 90, 160}
\definecolor{boxtitle}{RGB}{40, 90, 160}
\usepackage{tcolorbox}
\tcbuselibrary{skins, breakable}

\newtcolorbox{examplebox}[2][]{
  colback=white,
  colframe=boxframe,
  coltitle=white,
  fonttitle=\bfseries\large,
  title={#2},
  sharp corners=south,
  boxrule=0.5mm,
  drop shadow,
  breakable, 
  #1
}

\NewDocumentCommand{\yuji}
{ mO{} }{\textcolor{blue}{\textsuperscript{\textit{Yuji}}{\small[#1]}}}

\NewDocumentCommand{\tx}
{ mO{} }{\textcolor{teal}{\textsuperscript{\textit{Tianxin}}{\small[#1]}}}

\NewDocumentCommand{\heng}
{ mO{} }{\textcolor{red}{\textsuperscript{\textit{Heng}}\textsf{\textbf{\small[#1]}}}}

\NewDocumentCommand{\huan}
{ mO{} }{\textcolor{purple}{\textsuperscript{\textit{Huan}}\textsf{\textbf{\small[#1]}}}}

\usepackage{microtype}
\usepackage{graphicx}
\usepackage{subcaption}
\usepackage{booktabs} 

\usepackage{hyperref}

\usepackage[preprint]{icml2026}

\usepackage[capitalize,noabbrev]{cleveref}

\theoremstyle{plain}
\newtheorem{theorem}{Theorem}[section]
\newtheorem{proposition}[theorem]{Proposition}
\newtheorem{lemma}[theorem]{Lemma}
\newtheorem{corollary}[theorem]{Corollary}
\theoremstyle{definition}

\theoremstyle{remark}
\newtheorem{remark}[theorem]{Remark}

\usepackage[textsize=tiny]{todonotes}

\icmltitlerunning{Geometric-disentanglement Unlearning}

\begin{document}

\twocolumn[
  \icmltitle{Geometric-disentanglement Unlearning}



  \icmlsetsymbol{equal}{*}
  \begin{icmlauthorlist}
    \icmlauthor{Duo Zhou}{equal,uiuc}
    \icmlauthor{Yuji Zhang}{equal,uiuc}
    \icmlauthor{Tianxin Wei}{uiuc}
    \icmlauthor{Ruizhong Qiu}{uiuc}
    \icmlauthor{Ke Yang}{uiuc}
    \icmlauthor{Xiao Lin}{uiuc}
    \icmlauthor{Cheng Qian}{uiuc} \\
    \icmlauthor{Jingrui He}{uiuc}
    \icmlauthor{Hanghang Tong}{uiuc}
    \icmlauthor{Chengxiang Zhai}{uiuc}
    \icmlauthor{Heng Ji}{uiuc}
    \icmlauthor{Huan Zhang}{uiuc} 
  \end{icmlauthorlist}

  \icmlaffiliation{uiuc}{University of Illinois Urbana-Champaign. * for Equal Contribution}

  \icmlcorrespondingauthor{Duo Zhou, Yuji Zhang}{\{duozhou2, yujiz\}@illinois.edu}

  \icmlkeywords{Machine Learning, ICML}

  \vskip 0.3in
]



\printAffiliationsAndNotice{}  
\addtocontents{toc}{\protect\setcounter{tocdepth}{-1}}
\input{contents/abs}
\input{contents/1.intro}
\input{contents/2.prelim}

\input{contents/3.method}

\input{contents/4.exp}
\input{contents/5.related_work}
\input{contents/6.conclusion}

\bibliography{iclr2026_conference}
\bibliographystyle{icml2026}

\input{contents/appendix}

\end{document}

%% file: contents/abs.tex
\begin{abstract}
Large language models (LLMs) can internalize private or harmful content, motivating unlearning that removes a forget set while preserving retaining knowledge. However, forgetting updates often cause collateral degradation on retaining knowledge, creating a persistent trade-off. Existing LLM unlearning methods are often heuristic, and other theoretical approaches rely on offline feature constructions that do not capture update-time forget-retain interaction in LLMs.
To address this limitation, we aim to develop an LLM unlearning method that reduces the forget-retain trade-off with theoretical guarantees.
We take a first-principles view by formalizing ``no side effects'' as local retain invariance under small parameter updates, and prove an equivalence under optimizer-induced geometry: the retain loss is locally invariant if and only if the update direction is orthogonal to the subspace spanned by retain gradients. 
Based on the insight, we propose Geometric-disentanglement Unlearning (GU), a lightweight and theoretically grounded projection that can be plug-and-play to existing gradient-based unlearning methods to mitigate forget-retain side effects. Experiments on TOFU, MUSE, and WMDP-cyber show that GU strengthens forgetting while reducing retain drift. When added to SimNPO, it achieves up to 62\% improved forgetting Extraction Strength (ES) and 31\% higher retain ES. We open-sourced our code in \url{https://github.com/Lemutisme/Geometric-Unlearning}.

\end{abstract}

%% file: contents/1.intro.tex
\section{Introduction}

Large language models (LLMs) learn broad knowledge from massive corpora~\citep{touvron2023llama,grattafiori2024llama, wolf-etal-2020-transformers}, but this strength also creates deployment risk: models can internalize private or harmful content that later must be removed~\citep{carlini2021extracting,zhang2025atomic,li2024wmdp, zhang2024knowledge,  zhang-etal-2025-law}. Machine unlearning aims to modify a trained model so that the influence of a forget set is erased while performance on the retain set needs to be preserved~\citep{cao2015towards,bourtoule2021machine,ginart2019making,graves2021amnesiac}. In practice, however, updates that improve forgetting often degrade behavior on retaining data, revealing a persistent tradeoff between effective forgetting and retaining knowledge fidelity~\citep{openunlearning,le2025survey,chen-yang-2023-unlearn, yao2024machine}.



Some LLM unlearning methods attempt to mitigate the trade-off based on heuristic assumptions~\citep{liu2024large} about why side effects arise, for example, attributing them to entanglement measured by embedding similarity~\citep{microedit2024}, or relying on post-training heuristics to reduce harm to retained knowledge without a formal, testable specification~\citep{dong-etal-2025-undial,ji2024reversing, NPO_zhang2024negative}. 
Consequently, existing LLM unlearning approaches tend to lack theoretical guarantees for controlling the retain-forget trade-off, which could leave models either insufficiently scrubbed of harmful content or overly degraded on retained knowledge, increasing deployment risk.

To address this, some machine unlearning methods study the forget-retain trade-off through theoretical analysis. However, they often perform offline calibration of the overlap representation between forgetting and retaining data, such as label-space constructions of forget and retain features.
These designs neither apply to LLMs nor capture update-time collateral damage to retained knowledge~\citep{bourtoule2021machine,ginart2019making,sendera2025semu, zhang2023review,bourtoule2021machine}.
As a result, their applicability to reducing the forget-retain trade-off in LLM unlearning is limited, where principled, update-time calibration with theoretical guarantees remains an open challenge.

Motivated by this gap, we seek a simple and theoretically grounded LLM unlearning method that can reduce the forgetting-retaining trade-off.
The central question is: exactly under what conditions does a forgetting update cause side effects on retaining knowledge during unlearning, and can those effects be avoided with theoretical guarantees?
In fact, the ``no side effect'' condition yields a concrete standard for ``retain-invariant'' updates that do not affect the retain set during model updating. 
This motivates our first-principles objective: optimize the forgetting loss while keeping the retaining loss unchanged, enforced as a local ``retain-invariance'' requirement at update time. 
Rather than presupposing a representation of entanglement in embeddings or parameters, we work in the optimizer-induced geometry that characterizes the actual parameter update, capturing the true interaction between forgetting and retaining knowledge under preconditioned optimization like Adam~\citep{kingma2014adam}.
Concretely, we characterize which update directions satisfy retain-invariance via a first-order analysis of the retain loss change under small parameter updates.


This analysis yields a concrete and testable account of retain-forget interaction: the component of an update that is responsible for first-order harm on retained data. We prove a crisp equivalence: the retain loss is locally invariant if and only if the update direction is orthogonal, under the optimizer's geometry, to the subspace spanned by retain gradients. This characterization links disentanglement to orthogonality with the retain-gradient subspace. Accordingly, the trade-off arises from the tangential component of a forgetting update within this subspace, which perturbs the retain loss, and a retain-invariant forgetting update should therefore remove this tangential component.


Motivated by this insight, we introduce Geometric-disentanglement Unlearning (GU). GU constructs the orthogonal complement of the retain-gradient subspace and projects each forgetting update into this complement before applying it, preserving only the normal component that leaves the retain loss unchanged while removing the tangential component responsible for side effects. We show that, under a standard trust-region constraint, the projected direction most aligned with the raw forgetting gradient is optimal, achieving the steepest descent progress subject to local retain invariance. Additionally, from an optimization viewpoint, we derive the optimal joint update direction that balances the retain and forget gradients.
In addition, from an optimization perspective, we derive the optimal joint update direction of retain and forget gradients.

Built on a simple and sound theoretical guarantee, GU integrates easily into existing gradient-based unlearning pipelines: without altering core objectives or requiring additional regularizers, only conducting orthogonal projection from forget to retain gradients induced by \textit{existing} optimizer geometry, \textit{which to the best of our knowledge is the first plug-and-play LLM unlearning method with theoretical guarantees on updating-time trade-off alleviation}. 
Our experiments show that GU achieves stronger forgetting with smaller drift on the retain set, consistent with the theoretical link between reduced entanglement and orthogonality-based retain-invariance. Specifically, across three benchmarks using SimNPO~\citep{fan2024simplicity_simnpo}, recognized as the SOTA method~\citep{openunlearning},  adding our geometry-disentanglement unlearning yields up to 62\% improved forgetting Extraction Strength, 31\% higher retention Extraction Strength, and 8\% higher model utility, and 60\% higher MIA-closeness; on MUSE it improves Extraction Strength~\citep{carlini2021extracting} for Unlearning by 46\%, boosts retained ROUGE by 17\%, and reduces privacy-leak magnitude by 14\%; and on WMDP-cyber it lowers hazardous accuracy by 0.36\% without harming MMLU.

Taken together, adopting orthogonality to the retain gradient subspace as an explicit design principle provides a simple yet effective unified theoretical and practical framework for effective unlearning with controlled side effects. Our contributions are threefold:

$\bullet$ We formalize and leverage a theoretically sound equivalence that local retain invariance matches orthogonality to the retain-gradient subspace under optimizer geometry, thereby making side effects formally testable.

$\bullet$ We introduce Geometric-disentanglement Unlearning, a plug-and-play projection that provides a simple and lightweight theoretical and practical framework for unlearning with controlled side effects.

$\bullet$ Across three well-known benchmarks of TOFU, MUSE, and WMDP, GU strengthens forgetting while preserving or improving downstream performance.

%% file: contents/2.prelim.tex
\section{LLM Unlearning Preliminaries}
\label{sec:preliminaries}






\paragraph{Problem Definition.}
Let $\pi_\theta$ be the target model with $\theta\in\mathbb{R}^p$ denote the parameters of the model $\pi_{\theta}$, and $\pi_{\mathrm{ref}}$ be a reference model trained on a dataset $D$.
Real-world data may contain private or harmful samples. Let $D_f \subseteq D$
denote the forget subset whose influence must be removed, and define the retain set {$D_r = D \setminus D_f$}
. Starting from $\pi_{\mathrm{ref}}$, we continue
training to obtain our model $\pi_{\theta}$. Our objective is for $\pi_{\theta}$ to behave
as if $D_f$ had never been used, which is to say, to match the behavior of a model trained
from scratch on $D_r$. In principle, the ideal approach is full retraining on
$D_r$. However, in practice, this is often intractable due to heavy costs.
A common unlearning practice performs a bi-objective update at each step~\citep{maini2024tofu,openunlearning,NPO_zhang2024negative}. One samples a pair $\{x_f, x_r\}$ with $x_f \sim D_f$ and $x_r \sim D_r$. The update applies {forget loss such as gradient ascent} on a forget objective evaluated at $x_f$ and gradient descent on a retain objective evaluated at $x_r$. The intent is to forget information associated with $x_f$ while preventing unintended harm to $x_r$. We now make the two objectives explicit.
Generally, for both the forget and retain training strategies, there are many viable choices. Taking forget loss as an example, we consider the following instantiations.
\textit{Token-level NLL:} $\ell_f(x_f;\theta) \equiv$ sequence-averaged cross-entropy on $x_f$
(with a sign conventionally chosen for ascent/descent as needed).
\textit{Preference ratios (e.g., SimNPO~\citep{fan2024simplicity_simnpo}/ NPO~\citep{NPO_zhang2024negative}/DPO~\citep{rafailov2023dpo}):} use log-likelihood ratios against a frozen reference model
$\pi_{\mathrm{ref}}$ to penalize the originally preferred response and/or promote an alternative.
\textit{Calibration-based variants (e.g., CEU~\citep{yang2025u}/UNDIAL~\citep{dong-etal-2025-undial}/WGA~\citep{wang2025rethinking}/Sat-Imp~\citep{yang2025exploring}):} reshape logits or labels to
discourage reproducing forget content.
The loss can be instantiated in multiple 
ways, here we adopt the token-level NLL loss in the following practice for simplicity:
\vspace{-1em}
\paragraph{Forget loss.}
For a forget sample $x_f \in D_f$, let $\ell_f(x_f;\theta)$ denote a \emph{forget} objective that encourages the model to reject behaviors tied to $D_f$, optimized via \emph{gradient ascent}: 
\vspace{-0.5em}
\begin{equation}
\label{eq:forget-loss}
L_f(\theta):=-{\mathbb{E}_{x_f\sim D_f}\big[\ell_f(x_f;\theta)\big]}.
\vspace{-0.5em}
\end{equation}
\paragraph{Retain loss.}
Let $\ell_r(x_r;\theta)$ denote a retain objective that encourages the model to prefer behaviors tied to $D_r$, optimized via \emph{gradient descent}:
\begin{equation}
\label{eq:retain-loss}
L_r(\theta):= {\mathbb{E}_{x_r\sim D_r}\big[\ell_r(x_r;\theta)\big]}.
\vspace{-0.5em}
\end{equation}
\paragraph{Empirical objectives.}
We will form training objectives that combine (i) a forget term aggregated over $D_f$
and (ii) the retain-anchor $L_r$.
A generic empirical objective takes the form $\mathcal{L}_{joint}(\theta)=L_f(\theta)
+\alpha\,L_r(\theta)$, where $\alpha\ge 0$ balances forgetting and retention.

%% file: contents/3.method.tex
\section{Methodology}
\label{sec:method}
\begin{figure*}
    \centering
    \includegraphics[width=\linewidth]{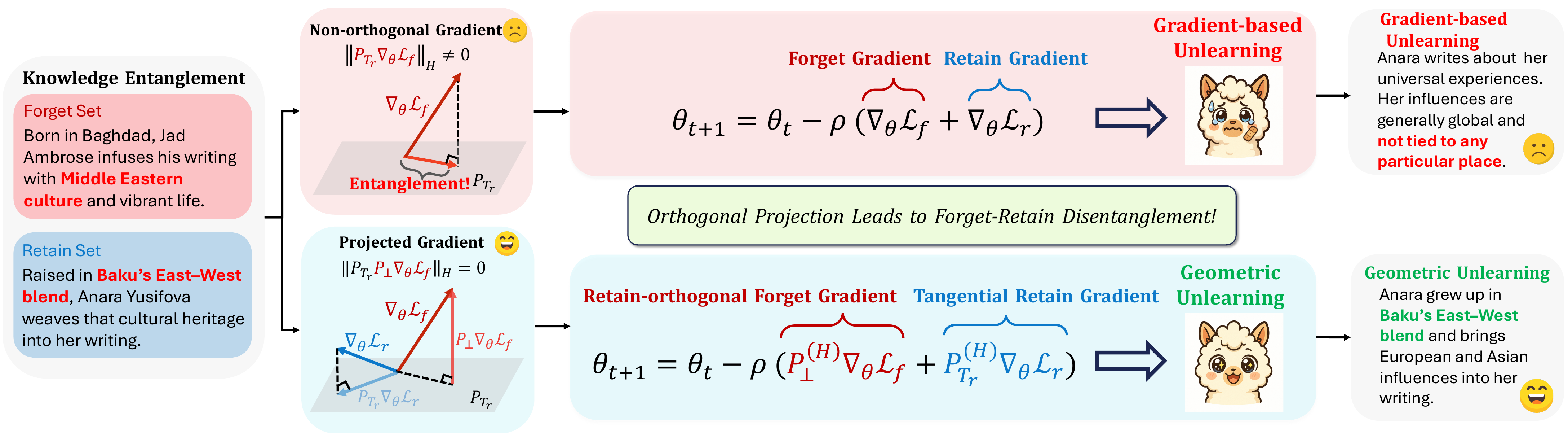}
    \caption{\textbf{GU (bottom) vs. baseline (top).} $P_{\perp}$ is the $H$-orthogonal projector onto the complement of retain tangent subspace $T_r$; $P_{T_r}$ projects onto $T_r$. Without changing training objective or adding regularization, we \emph{route} existing gradients through orthogonal projectors. } 
    \label{fig:gu}
    \vspace{-1em}
\end{figure*}


\subsection{Geometric-disentanglement Unlearning (GU)}

\paragraph{Featuring Side Effects on Retain Set.}

In LLM unlearning, updates that improve forgetting on $D_f$ can unintentionally harm knowledge on $D_r$. We attribute this trade-off to retain-forget entanglement. Prior efforts often pursue ``disentanglement'' via heuristics without theoretically deriving a formal, testable specification of what is being disentangled~\citep{liu-etal-2025-disentangling, sendera2025semu}. 
In contrast, to derive a theoretically rigorous forget-retain entanglement representation to mitigate the unlearning tradeoff accurately, we take a different route: starting from the desideratum \emph{forget reduced, retain unchanged}, which manifests during training as a \emph{local retain-invariance} requirement. Rather than presupposing a particular representation of entanglement (e.g., similarity in hidden states or in gradients), we first characterize the update directions that leave $L_r$ locally invariant. This characterization, in turn, induces a theoretically grounded representation of entanglement, namely, the component of an update that is accountable for the harm on $D_r$.
Concretely, during model training on a paired mini-step with forget and retain samples $\{x_f, x_r\}$ at iteration $t$, a parameter update is written as $\theta_{t+1}=\theta_t+\Delta\theta$, where $\Delta\theta\in\mathbb{R}^p$ is the step induced by the current optimization move.
When $\|\Delta\theta\|$ is small, retain loss $L_r$'s local change at $\theta_t$ along $\Delta\theta$ admits the first order approximation:
\begin{equation}
\Delta^{(1)} L_r = \langle \nabla_\theta L_r(x_r, \theta_t), \Delta\theta \rangle_H ,
\label{eq.first_order_approxi}
\end{equation}
where at training iteration $t$, we freeze an optimizer-induced symmetric positive definite (SPD) preconditioner $H_t \succ 0$ and equip $\mathbb{R}^p$ with the inner product $\langle u,v\rangle_{H_t}:=u^\top H_t v$ and norm $\|v\|_{H_t}:=\sqrt{\langle v,v\rangle_{H_t}}$. Within the iteration, including any line-search or trust-region computation, $H_t$ is treated as constant; it may be updated to $H_{t+1}$ at the next step.\footnote{This variable-metric view is standard: adaptive methods such as AdaGrad and Adam act as diagonal preconditioners and thus endow a stepwise SPD metric (see, e.g., \citet{duchi2011adaptive,kingma2014adam}; natural-gradient/K-FAC metrics~\citep{amari2019fisher,martens2015optimizing}). Discussion see Appendix~\ref{sec:discussion_adam}. 
}
Under this convention, the metric gradient with respect to $\langle\cdot,\cdot\rangle_{H_t}$ is
$\nabla^{H_t} L_r(x_r,\theta_t):=H_t^{-1}\nabla L_r(x_r,\theta_t)$,
and the first-order change along an update $\Delta\theta$ is
$\Delta^{(1)}L_r=\langle \nabla^{H_t} L_r(x_r,\theta_t),\Delta\theta\rangle_{H_t}$. We may omit the $t$ w.r.t $H$.
i.e., under the metric $H$, $\nabla_\theta L_r(x_r, \theta_t)$ is the gradient of retain sample $x_r$.
Our objective is to reduce this harm Eq.~\ref{eq.first_order_approxi}, ideally keeping $L_r$ unchanged, which at the local scale means enforcing $\Delta^{(1)} L_r=0$. To investigate when $\Delta^{(1)} L_r=0$, we propose the following proposition:

\begin{proposition}
\label{prop:retain-subspace-orth}
Formally, fix $\theta\in\mathbb{R}^p$ and an optimizer-induced symmetric positive definite metric $H\succ0$ with inner product
$\langle u,v\rangle_H := u^\top H v$ and its norm is $\|v\|_H := \sqrt{\langle v,v\rangle_H}$.
For each retain sample $x_r\in D_r$, assume $\ell_r(x_r;\theta)$ be differentiable and define retain gradient
$g(x_r):=\nabla_\theta \ell_r(x_r;\theta)\in\mathbb{R}^p$, 
Let the parameter-dependent retain gradient subspace be
\[
T_r(\theta):=\mathrm{span}\{g(x_r):x_r\in D_r\}\ \subseteq\ \mathbb{R}^p,
\]
and its $H$-orthogonal complement
$T_r(\theta)^\perp:=\{v\in\mathbb{R}^p:\ \langle v,g\rangle_H=0\ \forall g\in T_r(\theta)\}$.
For any finite collection $(x_{r_i})_{i=1}^m\subset D_r$, define the finite retain loss $L_r(\theta):=\sum_{i=1}^m \ell_r(x_{r_i};\theta)$,
$\nabla_\theta L_r(\theta)=\sum_{i=1}^mg(x_{r_i})\ \in\ T_r(\theta)$.
Then for any update direction $\Delta\theta\in\mathbb{R}^p$, the following are equivalent:

\quad(i) $\Delta\theta\in T_r(\theta)^\perp$. \quad
(ii) 
$\Delta^{(1)} L_r(\theta;\Delta\theta)
:= \langle \nabla_\theta L_r(\theta),\Delta\theta\rangle_H
= 0$ for all  $x_r \in D_r$.
\end{proposition}
i.e., when $L_r$ is locally invariant, $\Delta^{(1)} L_r=0$, if and only if the update direction $\Delta\theta$ is $H$-orthogonal to $T_r$.
This identifies $T_r^\perp$ as a \emph{retain-invariance} subspace for forgetting.
Proof see Appendix~\ref{sec:prooforth}.

\paragraph{Geometric Decomposition.}
As established above, if an update direction is $H$-orthogonal to the retain gradient subspace, then the retain loss $L_r$ is locally unchanged. 
This motivates a \emph{geometric} view: as shown in Fig.~\ref{fig:gu}, for a forget sample $x_f$, decompose its gradient
\begin{equation}
g_f(x_f)=P^{(H)}_{T_r}g_f(x_f)+P^{(H)}_{\perp}g_f(x_f),
\end{equation}
into a \emph{tangential} component $P^{(H)}_{T_r}g_f(x_f)\in T_r$ and a \emph{normal} component $P^{(H)}_{\perp}g_f(x_f)\in T_r^\perp$, where
\begin{equation}
\label{eq:proj-orth}
    P^{(H)}_{T_r}=U(U^\top H U)^{-1}U^\top H,\qquad P^{(H)}_{\perp}=I-P^{(H)}_{T_r},
\end{equation}
where $U=[u_1,\dots,u_k]\in\mathbb{R}^{p\times k}$ is the retain gradients from a small retain mini-batch $B_r\subset D_r$ on selected tensors, spans the retain gradient subspace $T_r=\mathrm{range}(U)$ with $U^\top H U=I$.
$I$ is the $k\times k$ identity matrix.
The normal component $P^{(H)}_{\perp}$ produces no first-order change on $L_r$ (discussion on the choice of retain anchor see Appendix~\ref{sec:discussion_retain}), while the tangential component $P^{(H)}_{T_r}$ captures the interaction with retain updates.
We therefore define retain-forget gradient update entanglement by the magnitude of the tangential component:
\begin{equation}
\operatorname{ent}_H\big(g_f(x_f)\big):=\big\|P^{(H)}_{T_r}g_f(x_f)\big\|_{H},
\label{eq.entanglement}
\end{equation}
which vanishes if and only if $g_f(x_f)\in T_r^\perp$.
In the disentangled case, $P_\perp^{(H)} g_f(x_f)$ yields a direction that is first-order safe for $L_r$.
However, the retain-invariance subspace $T_r^\perp$ contains infinite directions.
{Which} retain-invariance direction should we take under a fixed step budget for local optimal forgetting?
A first-order selection principle can answer this.
Fix the metric $H\succ0$,
fist-order linearizing the joint objective $\mathcal{L}_\text{joint}(\theta):= L_f(\theta)+\alpha L_r(\theta)$ at $\theta_t$ gives
\[
\Delta^{(1)}\mathcal{L}_\text{joint}(\theta_t;\Delta\theta)
=\langle \nabla^H L_f(\theta_t)+\alpha\nabla^H L_r(\theta_t),\Delta\theta\rangle_H.
\]
Because $\nabla^H L_r(\theta_t)\in T_r$ and $\Delta\theta\in T_r^\perp$, the retain term vanishes:
$\langle \nabla^H L_r(\theta_t),\Delta\theta\rangle_H=0$.
Hence, \emph{within the retain-invariance set}, the steepest first-order change of $\mathcal{L}_\text{joint}$ coincides with that of $L_f$,
and depends only on $g_f:=\nabla^H L_f(\theta_t)$ projected onto $T_r^\perp$.
This leads to the following lemma: 


\begin{lemma}[Steepest feasible descent under first-order safety]
\label{lem:steepest}
Let $H\succ0$ and let $T_r=\mathrm{range}(U)$ be the retain-gradient subspace
(with respect to the $H$-inner product). Define the feasible set
\begin{align*}
    \mathcal C\ :=&\ \{\Delta\theta\in\mathbb{R}^p:\ U^\top H\Delta\theta=0,\ \|\Delta\theta\|_H\le1\}
\\ =&\ \{v\in T_r^\perp:\ \|v\|_H\le1\}.
\end{align*}
For $g_f:=\nabla_\theta^H L_f(\theta_t)$, the direction achieving the largest first-order decrease of $L_f$ over $\mathcal C$ is
\vspace{-1mm}
\begin{equation*}
    \Delta\theta_f^\star
= \arg\min_{\Delta\theta\in\mathcal C}\ \langle g_f,\Delta\theta\rangle_H
= -\frac{P_\perp^{(H)} g_f}{\|P_\perp^{(H)} g_f\|_H},
\end{equation*}
\vspace{-1mm}
$\Delta\theta_f^\star$ is unique if $P_\perp^{(H)}g_f\neq 0$. Moreover, letting $g_r:=\nabla_\theta^H L_r(\theta_t)\in T_r$, the same $\Delta\theta_f^\star$ also achieves the largest first-order decrease of the joint objective
$\mathcal L_{\rm joint}:=L_f+\alpha L_r$ over $\mathcal C$:
\[
\Delta\theta_f^\star
= \arg\min_{\Delta\theta\in\mathcal C}\ \big\langle g_f+\alpha g_r,\ \Delta\theta\big\rangle_H .
\]
\end{lemma}
Proof see Appendix~\ref{sec:prooflem:steepest} for details. This provides the optimal update step for the total loss $\mathcal{L}_{joint}$:
\begin{equation}
    \label{eq:split-update}
    \theta_{t+1}
=\theta_t-\rho\Big(
\underbrace{P_{\perp}^{(H)} \nabla^H L_f(\theta_t)}_{\text{retain-orthogonal}}
+\underbrace{P_{T_r}^{(H)} \nabla^H L_r(\theta_t)}_{\text{retain-tangent}}
\Big)
\end{equation}
is first-order optimal for the joint objective under the retain-safety constraint.
Under standard $H$-smoothness, the step size $\rho$ can be selected by a trust-region or line-search rule
(see \S~\ref{sec:theory}, Proposition~\ref{prop:safety}, Corollary~\ref{cor:descentLr}).
Noted that our geometric-disentanglement update is a plug-and-play method, and projection touches only selected trainable tensors, making GU architecture-agnostic.
We present the details of Algorithm~\ref{alg:GU} in Appendix~\ref{sec:alg} for the basis calculation and the optimizer update step. \noindent\textbf{Does projection weaken forgetting? No:}
GU removes only the retain-interfering component of the forget update; unless the forget gradient lies entirely in the retain-gradient subspace, the projected direction remains a strict first-order descent direction for the forget loss (see Appendix~\ref{sec:appendix_projection_not_weaken}).

\subsection{Theoretical Guarantees}

\label{sec:theory}

We provide the theoretical guarantees for our method, GU. We have proven that $P^{(H)}_\perp g_f$ is the \emph{steepest} safe direction for $\mathcal{L}_{joint}$ in Lemma~\ref{lem:steepest}, furthermore, we will show $L_r$ is \emph{first-order nonincreasing} (strictly decreasing when $\beta>0$), with second-order drift bounded by smoothness in Prop.~\ref{prop:safety} and its corollary Cor.~\ref{cor:descentLr}).
Then, we show that the composite objective enjoys a \emph{nonpositive} first-order change with an explicit negative lower bound in Prop.~\ref{prop:composite-exact}.
Collectively, these results justify GU as a principled \emph{first-order safe} and \emph{steepest-feasible} unlearning procedure in the optimizer geometry, with explicit stability and robustness margins.


\paragraph{First-Order Safety and Retain Monotonicity}

The next proposition quantifies, at first order, how this step impacts the retain loss $L_r$:
the normal forget component is first-order neutral to $L_r$, whereas the tangential repair strictly
decreases $L_r$ whenever $g_r\neq0$.

\begin{proposition}[First-order safety and retain monotonicity]
\label{prop:safety}
Let $H\succ0$ be SPD and let $T_r\subset\mathbb{R}^p$ denote the retain-gradient subspace w.r.t. the $H$-inner product.
Let $g_r:=\nabla_\theta^H L_r(\theta_t)\in T_r$ and $g_f:=\nabla_\theta^H L_f(\theta_t)$.
WLOG, introduce $\beta \ge 0$. Consider one split step
\begin{equation}
\Delta\theta = -\rho\Big(P^{(H)}_\perp g_f + \beta P^{(H)}_{T_r} g_r\Big)
\quad\text{with}\quad \rho>0,\ \beta\ge0.
\end{equation}
Then the first-order change of $L_r$ satisfies
\begin{equation}
\Delta^{(1)} L_r
=
\langle g_r,\Delta\theta\rangle_H
=
-\rho\beta\|g_r\|_H^2
\le0.
\end{equation}
If $\beta=0$ the step is \emph{first-order neutral} to $L_r$, and if $\beta>0, g_r\neq0$
it is \emph{first-order strictly decreasing}.
\end{proposition}

Proposition~\ref{prop:safety} establishes the \emph{first-order} effect of one split step on the retain loss:
$\Delta^{(1)}L_r=\langle g_r,\Delta\theta\rangle_H=-\rho\beta\|g_r\|_H^2\le0$.
To convert this into an \emph{actual} decrease of $L_r(\theta)$, we invoke the $H$-geometry version of the descent lemma under Lipschitz $H$-gradient, and combine it with the $H$-orthogonal decomposition of the step:

\begin{corollary}[Descent guarantee for $L_r$ under $H$-smoothness]
\label{cor:descentLr}
Assume the $H$-gradient $\nabla_\theta^H L_r$ is $L_r^{(H)}$-Lipschitz under $\|\cdot\|_H$, i.e.,
$\|\nabla_\theta^H L_r(\theta+\Delta)-\nabla_\theta^H L_r(\theta)\|_H\le L_r^{(H)}\|\Delta\|_H$.
Let the split step be $\Delta\theta=-\rho\big(P_\perp^{(H)}g_f+\beta P_{T_r}^{(H)}g_r\big)$ with $\rho>0$ and $\beta\ge0$, where $g_r:=\nabla_\theta^H L_r(\theta)$ and $g_f:=\nabla_\theta^H L_f(\theta)$.
Then
\begin{align}
L_r(\theta+\Delta\theta)
\ &\le\
L_r(\theta)\ -\ \rho \beta \|g_r\|_H^2 
\\ &+\ \frac{L_r^{(H)}}{2} \rho^2\Big( \|P_\perp^{(H)}g_f\|_H^2+\beta^2\|g_r\|_H^2 \Big) \notag.
\end{align}
In particular, if $0\ <\ \rho\ <\ \frac{2 \beta \|g_r\|_H^2}{ L_r^{(H)}\big( \|P_\perp^{(H)}g_f\|_H^2+\beta^2\|g_r\|_H^2 \big)}$ , then $L_r(\theta+\Delta\theta)<L_r(\theta)$ (strict descent whenever $\beta>0$ and $g_r\neq0$).
For $\beta=0$,
\begin{align}
    L_r(\theta+\Delta\theta)\ &\le\ L_r(\theta)\ +\ \frac{L_r^{(H)}}{2} \rho^2 \|P_\perp^{(H)}g_f\|_H^2 \notag
\\ &=\ L_r(\theta)+O(\rho^2),
\end{align}
recovering the neutral first-order case with only second-order drift.
\end{corollary}
Proof details of Proposition~\ref{prop:safety} and Corollary~\ref{cor:descentLr} in Appendix~\ref{sec:proofsafety}.



\paragraph{One-step behavior of the joint objective.}
Having established first-order monotonicity and actual descent for $L_r$, we now analyze the one-step
first-order change of the joint objective $\mathcal L_{\text{joint}}:=L_f+\alpha L_r$ under the same split step.

\begin{proposition}[Exact first-order change of $\mathcal L_{\text{joint}}$]
\label{prop:composite-exact}
Let $H\succ0$, $g_f:=\nabla_\theta^H L_f(\theta)$, $g_r:=\nabla_\theta^H L_r(\theta)\in T_r$, and
$\Delta\theta=-\rho\big(P_\perp^{(H)}g_f+\beta P_{T_r}^{(H)}g_r\big)$ with $\rho>0$, $\beta\ge0$.
Then the first-order change of the joint objective equals
\begin{align}
    \label{eq:joint-first-order-identity}
    \Delta&^{(1)}\mathcal L_{\text{joint}}
    :=\big\langle g_f+\alpha g_r,\Delta\theta\big\rangle_H\\
    =& - \rho\Big( \|P_\perp^{(H)}g_f\|_H^2+\alpha\beta\|g_r\|_H^2+\beta\langle P_{T_r}^{(H)}g_f, g_r\rangle_H \Big). \notag
\end{align}
\end{proposition}
Proof details see Appendix~\ref{sec:proofcomposite}. In the optimizer-induced metric $H$, we prove that GU performs first-order-safe, steepest-feasible forgetting by projecting onto the retain-orthogonal subspace, guarantees monotone decrease of the retain loss via an explicit stepsize condition, provides an exact one-step decomposition for the joint objective with verifiable nonpositivity conditions, and quantifies retain-forget entanglement by the norm of the tangential $\|P^{(H)}_{T_r}g_f\|_H$.

%% file: contents/4.exp.tex
\vspace{-1em}
\section{Experiments}
\vspace{-0.5em}
\subsection{Experimental Settings}
\input{table/tofu}

\paragraph{Datasets}
\vspace{-0.5em}
We evaluate our method on the \textbf{Open Unlearning}~\cite{openunlearning} benchmark suite, focusing primarily on \textbf{TOFU}~\citep{openunlearning}, a fine-grained benchmark with 200 fictitious author profiles, each containing 20 QA pairs. For fair comparison, we adopt the Llama-3 backbones (1B, 3B, 8B) \citep{grattafiori2024llama} provided by the suite and follow the official \emph{scaling splits}, varying the forget set size (\texttt{forget01}, \texttt{forget05}, \texttt{forget10}) to examine scalability. In addition, we report results on \textbf{MUSE}~\citep{shi2025muse}, which evaluates memorization and unlearning of books and news articles through verbatim reproduction, question answering, and membership inference, and on \textbf{WMDP}~\citep{li2024wmdp}, an alignment-oriented benchmark of 3,668 multiple-choice questions across hazardous domains (biosecurity, cybersecurity, chemical security) assessing whether models can forget dangerous capabilities while retaining general performance. For MUSE and WMDP, we report results on Llama-2-7B \citep{touvron2023llama} and zephyr-7b \citep{tunstall2024zephyr} to provide a more comprehensive evaluation. 

\begin{figure*}[!htbp]
    \centering
    \includegraphics[width=\linewidth]{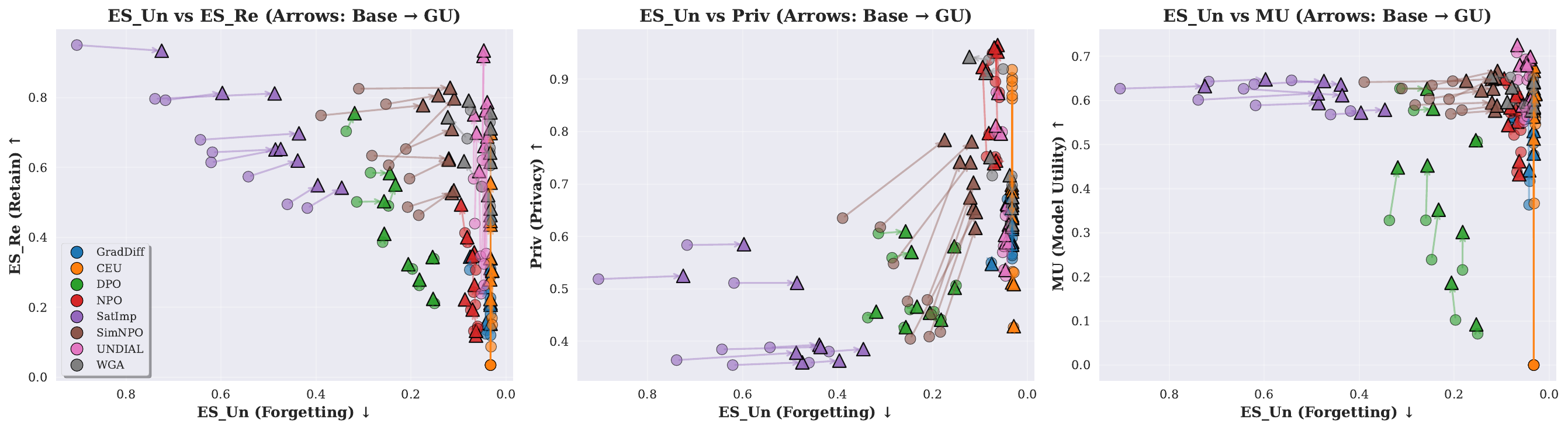}
    \vspace{-2em}
    \caption{{We visualize forgetting quality (ES Un: lower for better) against retained knowledge (ES Re), privacy (Priv), and model utility (MU) for eight unlearning baselines on TOFU. ES Re, Priv, and MU are metrics of higher for better. Circles denote baseline outputs, triangles denote results of GU, and arrows indicate the shift from Base $\rightarrow$ GU. Across all three panels, GU pushes methods toward the Pareto-optimal corner (upper-right), reducing the trade-off between forgetting and retaining.}}
    \label{fig:pareto_TOFU}
    \vspace{-5mm}
\end{figure*}

\vspace{-1em}
\paragraph{Evaluation Metrics}
Following \citet{openunlearning,yang2025exploring}, we evaluate unlearning performance along four axes. \textbf{Forgetting} is measured by \emph{Extraction Strength on the forget set} (ES,Un.; $\downarrow$), which quantifies residual regurgitation by testing how easily the model can reconstruct target facts under constrained prompts, directly probing whether the intended knowledge has been removed. \textbf{Retention} is assessed by \emph{Extraction Strength~\citep{carlini2021extracting} on the retain set} (ES,Re.; $\uparrow$), monitoring collateral damage to preserved knowledge; for MUSE and WMDP, this is complemented with \emph{ROUGE RE}, which measures generation quality on retained knowledge-based QA pairs. \textbf{Privacy} is captured by resistance to membership inference: on TOFU we report \emph{MIA closeness} ($\uparrow$), which evaluates the similarity between unlearned and retain-only models across multiple MIA variants, while on MUSE and WMDP we use \emph{Privacy Leakage (Priv. Leak.; $\downarrow$)}, which directly tests whether membership information from the forget set can still be inferred. Finally, \textbf{Utility} ($\uparrow$) captures post-unlearning usefulness: TOFU reports a composite model-utility score combining probability, ROUGE, and Truth Ratio across retain and factual knowledge sets, while MUSE and WMDP report ROUGE on retained QA tasks. Together, these axes disentangle \emph{what was forgotten} (ES,Un.), \emph{what was preserved} (ES,Re., ROUGE RE), \emph{whether leakage is controlled} (MIA, Priv. Leak.), and \emph{whether the model remains useful} (Utility).{A detailed introduction of the metrics refer to \textbf{Open Unlearning} \citep{openunlearning}.}

\input{table/muse}
\vspace{-1em}
\subsection{Results: GU Delivers Pareto Improvements}

Table~\ref{tab:main_results} shows the results on the \textbf{TOFU} benchmark. We use two informative references. \emph{Vanilla} is the pretrained backbone without any TOFU fine-tuning; it neither learns nor regurgitates TOFU facts (low ES on both splits), enjoys perfect MIA-closeness ($\mathrm{Priv}=1$), and yields moderate utility. \emph{Fully-finetuned} is trained on the entire TOFU corpus; it memorizes broadly (high ES on both splits), collapses privacy ($\mathrm{Priv}=0$), and reaches a utility ceiling. The practical goal is to move unlearning methods off these single-orbit extremes toward a frontier that combines low ES on the forget set with high ES on the retain set and high utility, while keeping privacy nontrivial.

Table~\ref{tab:muse_benchmark} shows the results on two complementary settings: (i) \textbf{MUSE} (Llama-2-7B) probes verbatim reproduction and QA over Books/News, where forgetting should \emph{reduce} ES on the forget split while \emph{preserving} ROUGE on retained QA and \emph{reducing} privacy leakage toward zero; (ii) \textbf{WMDP-cyber} (zephyr-7b-beta) probes capability removal, where lower unlearning accuracy (Un.\ acc.\,$\downarrow$) signals safer behavior while general ability (MMLU\,$\uparrow$) should not degrade.
\vspace{-1.5em}
\paragraph{Observed Pareto shifts at fixed or lower ES Un.}
{In Figure~\ref{fig:pareto_TOFU}, the Base $\rightarrow$ GU shifts follow a Pareto-improving direction across nearly all metrics. Specifically, (i) ES Un decreases further, indicating that GU not only preserves but even slightly improves forgetting effectiveness; (ii) ES Re and Priv increase, demonstrating that GU substantially mitigates the trade-off typically observed in existing unlearning methods, forgetting the target knowledge no longer harms retained knowledge or privacy; and (iii) MU remains stable or improves, showing that GU enhances unlearning without compromising overall model utility. Taken together, these trends highlight that GU enables precise and low-side-effect unlearning, transforming the unlearning process from a severe trade-off challenge into a near-Pareto-optimal operation.}
These shifts match the geometric expectation that removing retain-tangent components preserves forgetting while unlocking retention and utility. The more detailed analysis on TOFO and the Pareto improvement and detailed analysis on MUSE and WMDP are in Appendix~\ref{sec:tofuanalysis},\ref{sec:museanalysis}.
\vspace{-1em}
\subsection{Ablations, Cost Analysis and Qualitative Analysis}
\vspace{-0.5em}
Additional ablations on the choice of geometry (Adam-diagonal vs. Euclidean), optimizer compatibility (AdamW vs. SGD), multi-seed robustness, and a profiling-based cost breakdown (showing $<3\%$ wall-clock overhead and $<5\%$ memory increase) are provided in Appendix~\ref{app:ablation-cost}. Qualitative analysis and case studies are detailed in Appendix~\ref{sec:qualitative}.


%% file: table/tofu.tex
\begin{table*}[!htbp]
    \centering
    \vspace{-6pt}
    \caption{TOFU results comparing unlearning objectives w. and w/o GU. Arrows $\uparrow/\downarrow$ denote that higher/lower is better. Within each block (model scale and deletion rate), the top two entries are shaded: \colorbox{myblue}{blue} for higher-is-better metrics and \colorbox{mypink}{red} for lower-is-better metrics (All blocks with color \colorbox{myblue}{blue} or \colorbox{mypink}{red} means ours improve). Abbreviations: ES Re. = Extraction Strength on the retain split; ES Un. = Extraction Strength on the forget split; Priv. = MIA closeness; MU = composite model utility. “Forget–1\%, 5\%, 10\%” indicates the fraction of TOFU authors deleted. “Vanilla” is the pretrained backbone without TOFU fine-tuning; “fully-finetuned” is trained on the full TOFU corpus. “w.\ GU” denotes the corresponding objective augmented with our geometry module.}
    \vspace{-0.5em}
    \label{tab:main_results}
    \resizebox{0.99\textwidth}{!}{
    \begin{tabular}{ccccc|cccc|cccc}
    \toprule[1.5pt]
    \multirow{2}{*}{Method} & \multicolumn{4}{c|}{Forget-1\%} & \multicolumn{4}{c|}{Forget-5\%} & \multicolumn{4}{c}{Forget-10\%}\\
    \cmidrule(lr){2-5} \cmidrule(lr){6-9} \cmidrule(lr){10-13}
    &ES Re. $\uparrow$ & ES Un. $\downarrow$ & Priv. $\uparrow$ & MU $\uparrow$ &ES Re. $\uparrow$ & ES Un. $\downarrow$ & Priv. $\uparrow$ & MU $\uparrow$ &ES Re. $\uparrow$ & ES Un. $\downarrow$ & Priv. $\uparrow$ & MU $\uparrow$\\
    \midrule[1.5pt]
    \multicolumn{13}{c}{Llama-3.2-1B-Instruct}\\
    \midrule[1.2pt]
    Vanilla & 0.0657 & 0.0692 & 1.0 & 0.5986 & 0.0667 & 0.0634 & 1.0 & 0.5991& 0.0672 & 0.0589 &  1.0 & 0.5911 \\
    fully-finetuned & 0.6483 &0.7431 & 0.0 & 0.5991 & 0.6547  & 0.7271 & 0.0 & 0.5991 & 0.6475 & 0.7062 & 0.0& 0.5991 \\
    \midrule[0.2pt]
    GradDiff & 0.1347 & 0.0410 & 0.6478 & 0.4170 & 0.2024 & 0.0327 & 0.6619 & 0.5232 & 0.1202 & 0.0325 & 0.5576 & 0.4763\\
    GradDiff w. GU  & \cellcolor{myblue!50}0.1558 & \cellcolor{mypink!1}0.0421 & \cellcolor{myblue!15}0.6598 & \cellcolor{myblue!62}0.4417 & \cellcolor{myblue!20}0.2125 & \cellcolor{mypink!7}0.0327 & \cellcolor{myblue!15}0.6661 & \cellcolor{myblue!33}0.5308 & \cellcolor{myblue!50}0.1531 & \cellcolor{mypink!15}0.0325 & \cellcolor{myblue!34}0.5897 & \cellcolor{myblue!45}0.4798 \\

    CEU & 0.0875 & 0.0316 & 0.5328 & 0.3666 & 0.0348 & 0.0327 & 0.8855 & 0.0000 & 0.0348 & 0.0325 & 0.9022 & 0.0000\\
    CEU w. GU  & \cellcolor{myblue!100}0.2236 & \cellcolor{mypink!1}0.0328 & \cellcolor{mypink!1}0.5121 & \cellcolor{myblue!100}0.5134 & \cellcolor{myblue!100}0.2798 & \cellcolor{mypink!4}0.0333 & \cellcolor{mypink!2}0.6986 & \cellcolor{myblue!100}0.5635 & \cellcolor{myblue!100}0.4366 & \cellcolor{mypink!15}0.0325 & \cellcolor{mypink!2}0.6598 &\cellcolor{myblue!100} 0.5844 \\
    
    DPO & 0.3391 & 0.1520 & 0.5788 & 0.5071 & 0.2114 & 0.1507 & 0.5065 & 0.0710 & 0.2629 & 0.1826 & 0.4412 & 0.2157\\
    DPO w. GU  & \cellcolor{myblue!50}0.3440 & \cellcolor{mypink!1}0.1545 & \cellcolor{myblue!15}0.5813 & \cellcolor{myblue!15}0.5099 & \cellcolor{myblue!20}0.2243 & \cellcolor{mypink!1}0.1535 & \cellcolor{mypink!1}0.5020 & \cellcolor{myblue!20}0.0922 & \cellcolor{myblue!30}0.2792 & \cellcolor{mypink!25}0.1822 & \cellcolor{myblue!1}0.4411 & \cellcolor{myblue!100}0.3016 \\
    NPO & 0.3071 & 0.0637 & 0.7989 & 0.5482 & 0.1321 & 0.0678 & 0.8954 & 0.4378 & 0.1924 & 0.0742 & 0.9491 & 0.5218\\
    NPO w. GU  & \cellcolor{myblue!56}0.3574 & \cellcolor{mypink!1}0.0670 & \cellcolor{myblue!85}0.9595 & \cellcolor{myblue!27}0.5520 & \cellcolor{mypink!2}0.1191 & \cellcolor{mypink!22}0.0632 & \cellcolor{myblue!100}0.9651 & \cellcolor{myblue!69}0.4623 & \cellcolor{myblue!50}0.2226 & \cellcolor{mypink!2}0.0864 & \cellcolor{mypink!0}0.9172 & \cellcolor{myblue!40}0.5442 \\
    SatImp & 0.6437 & 0.6183 & 0.5112 & 0.5889 & 0.4948 & 0.4604 & 0.3591 & 0.5682 & 0.4841 & 0.4184 & 0.3804 & 0.5760\\
    SatImp w. GU  & \cellcolor{myblue!44}0.6517 & \cellcolor{mypink!100}0.4855 & \cellcolor{myblue!15}0.5114 & \cellcolor{myblue!28}0.5942 & \cellcolor{myblue!56}0.5494 & \cellcolor{mypink!55}0.3964 & \cellcolor{myblue!26}0.3632 & \cellcolor{myblue!30}0.5724 & \cellcolor{myblue!72}0.5423 & \cellcolor{mypink!40}0.3459 & \cellcolor{myblue!22}0.3850 & \cellcolor{myblue!27}0.5790 \\
    
    SimNPO & 0.6341 & 0.2824 & 0.5482 & 0.5899 & 0.4868 & 0.2072 & 0.4089 & 0.5696 & 0.4636 & 0.1838 & 0.4178 & 0.5781\\
    SimNPO w. GU  & \cellcolor{myblue!1}0.6260 & \cellcolor{mypink!89}0.1204 & \cellcolor{myblue!70}0.7414 & \cellcolor{myblue!27}0.5954 & \cellcolor{myblue!38}0.5272 & \cellcolor{mypink!57}0.1140 & \cellcolor{myblue!100}0.6540 & \cellcolor{myblue!31}0.5770 & \cellcolor{myblue!83}0.5350 & \cellcolor{mypink!65}0.1099 & \cellcolor{myblue!100}0.6163 & \cellcolor{myblue!35}0.5884 \\
    
    UNDIAL & 0.3462 & 0.0539 & 0.7994 & 0.5512 & 0.2391 & 0.0524 & 0.5697 & 0.5567 & 0.2631 & 0.0463 & 0.5246 & 0.5645\\
    UNDIAL w. GU  & \cellcolor{myblue!100}0.5900 & \cellcolor{mypink!1}0.0565 & \cellcolor{mypink!1}0.7962 & \cellcolor{myblue!37}0.5886 & \cellcolor{myblue!100}0.6613 & \cellcolor{mypink!30}0.0458 & \cellcolor{myblue!26}0.5888 & \cellcolor{myblue!47}0.5972 & \cellcolor{myblue!100}0.6888 & \cellcolor{mypink!20}0.0395 & \cellcolor{myblue!42}0.5889 & \cellcolor{myblue!74}0.6026 \\
    
    WGA & 0.5455 & 0.0516 & 0.9194 & 0.5872 & 0.4891 & 0.0335 & 0.7152 & 0.5836 & 0.4474 & 0.0325 & 0.6685 & 0.5825\\
    WGA w. GU  & \cellcolor{myblue!74}0.6180 & \cellcolor{mypink!0}0.0884 & \cellcolor{mypink!1}0.9119 & \cellcolor{myblue!28}0.5963 & \cellcolor{myblue!35}0.5212 & \cellcolor{mypink!1}0.0377 & \cellcolor{myblue!13}0.7168 & \cellcolor{mypink!1}0.5773 & \cellcolor{myblue!41}0.4828 &\cellcolor{mypink!15} 0.0325 & \cellcolor{myblue!36}0.6752 & \cellcolor{myblue!24}0.5862 \\
    \midrule[1.2pt]
    \multicolumn{13}{c}{Llama-3.2-3B-Instruct}\\
    \midrule[1.2pt]

    Vanilla & 0.0689& 0.0647 & 1.0 & 0.0649 & 0.0694 & 0.0656 & 1.0 & 0.6594 & 0.0645 & 0.0665 & 1.0 & 0.6623\\
    fully-finetuned  & 0.8763  & 0.9201 & 0.0 & 0.6660 & 0.8459 & 0.8869 & 0.0 & 0.6660 &  0.8730 &  0.8904  & 0.0 &  0.6660\\
    \midrule[0.2pt]

    GradDiff & 0.1241 & 0.0425 & 0.6712 & 0.3635 & 0.2273 & 0.0327 & 0.5974 & 0.5031 & 0.1808 & 0.0325 & 0.6340 & 0.5720\\
    GradDiff w. GU  & \cellcolor{myblue!100}0.2578 & \cellcolor{mypink!1}0.0436 & \cellcolor{mypink!1}0.6625 & \cellcolor{myblue!100}0.5677 & \cellcolor{myblue!80}0.3123 & \cellcolor{mypink!10}0.0327 & \cellcolor{myblue!13}0.6145 & \cellcolor{myblue!43}0.5822 & \cellcolor{myblue!40}0.2074 &\cellcolor{mypink!10} 0.0325 & \cellcolor{myblue!0}0.5842 & \cellcolor{myblue!45}0.6041 \\

    CEU & 0.1692 & 0.0297 & 0.4265 & 0.5585 & 0.0348 & 0.0327 & 0.8963 & 0.0000 & 0.0348 & 0.0325 & 0.8627 & 0.0000\\
    CEU w. GU  & \cellcolor{myblue!100}0.3046 & \cellcolor{mypink!29}0.0291 & \cellcolor{myblue!14}0.4288 & \cellcolor{myblue!70}0.6159 & \cellcolor{myblue!100}0.3411 & 0.0332 & \cellcolor{mypink!2}0.6924&\cellcolor{myblue!100} 0.6255 & \cellcolor{myblue!100}0.5568 & \cellcolor{mypink!7}0.0325 & \cellcolor{mypink!2}0.6350 & \cellcolor{myblue!100} 0.6672 \\
    DPO & 0.5017 & 0.3143 & 0.6057 & 0.6273 & 0.3098 & 0.1973 & 0.4564 & 0.1023 & 0.3866 & 0.2598 & 0.4264 & 0.3283\\
    DPO w. GU  & \cellcolor{myblue!40}0.5035 & \cellcolor{mypink!60}0.2569 & \cellcolor{mypink!1}0.6097 & \cellcolor{myblue!0}0.6266 & \cellcolor{myblue!50}0.3233 & \cellcolor{mypink!1}0.2058 & \cellcolor{mypink!1}0.4544 & \cellcolor{myblue!100}0.1867 & \cellcolor{myblue!40}0.4104 & \cellcolor{mypink!14}0.2564 & \cellcolor{myblue!24}0.4270 & \cellcolor{myblue!100}0.4526 \\

    NPO & 0.4129 & 0.0865 & 0.7520 & 0.6356 & 0.1454 & 0.0595 & 0.8653 & 0.4828 & 0.1389 & 0.0600 & 0.8746 & 0.5329\\
    NPO w. GU  & \cellcolor{myblue!66}0.4941 & \cellcolor{mypink!1}0.0947 & \cellcolor{myblue!100}0.9232 & \cellcolor{myblue!27}0.6499 & \cellcolor{myblue!0}0.1329 & \cellcolor{myblue!0}0.0632 & \cellcolor{myblue!66}0.9526 & \cellcolor{myblue!0}0.4326 & \cellcolor{myblue!70}0.1936 & \cellcolor{myblue!0}0.0704 & \cellcolor{myblue!82}0.9604 & \cellcolor{myblue!40}0.5805 \\
    SatImp & 0.7926 & 0.7171 & 0.5836 & 0.6429 & 0.6147 & 0.6210 & 0.3546 & 0.6370 & 0.5739 & 0.5426 & 0.3884 & 0.6457\\
    SatImp w. GU  & \cellcolor{myblue!34}0.8134 & \cellcolor{mypink!77}0.5973 & \cellcolor{myblue!23}0.5852 & \cellcolor{myblue!28}0.6480 & \cellcolor{myblue!46}0.6529 & \cellcolor{mypink!70}0.4744 & \cellcolor{myblue!26}0.3598 & \cellcolor{myblue!30}0.6441 & \cellcolor{myblue!72}0.6200 & \cellcolor{mypink!77}0.4386 & \cellcolor{myblue!22}0.3934 & \cellcolor{myblue!0}0.6362 \\
    SimNPO & 0.7490 & 0.3896 & 0.6349 & 0.6417 & 0.6068 & 0.2470 & 0.4046 & 0.6342 & 0.5682 & 0.2032 & 0.4511 & 0.6439\\
    SimNPO w. GU  & \cellcolor{myblue!56}0.7781 & \cellcolor{mypink!89}0.1747 & \cellcolor{myblue!70}0.7842 & \cellcolor{myblue!27}0.6447 & \cellcolor{myblue!98}0.7977 & \cellcolor{mypink!100}0.1089 & \cellcolor{myblue!100}0.6467 & \cellcolor{myblue!31}0.6670 & \cellcolor{myblue!83}0.6224 & \cellcolor{mypink!65}0.1207 & \cellcolor{myblue!100}0.6745 & \cellcolor{myblue!35}0.6488 \\
    UNDIAL & 0.4396 & 0.0658 & 0.8748 & 0.6468 & 0.3242 & 0.0465 & 0.6397 & 0.6463 & 0.3538 & 0.0416 & 0.5833 & 0.6550\\
    UNDIAL w. GU  & \cellcolor{myblue!100}0.6996 & \cellcolor{mypink!24}0.0619 & \cellcolor{myblue!1}0.8736 & \cellcolor{myblue!57}0.6805 & \cellcolor{myblue!100}0.7641 & \cellcolor{mypink!30}0.0424 & \cellcolor{myblue!26}0.6616 & \cellcolor{myblue!47}0.6935 & \cellcolor{myblue!100}0.7869 & \cellcolor{mypink!21}0.0396 & \cellcolor{myblue!22}0.6226 & \cellcolor{myblue!44}0.6992 \\
    WGA & 0.6827 & 0.0818 & 0.9365 & 0.6522 & 0.6060 & 0.0327 & 0.6975 & 0.6417 & 0.6427 & 0.0341 & 0.6516 & 0.6497\\
    WGA w. GU  & \cellcolor{myblue!74}0.7440 & \cellcolor{mypink!1}0.1226 & \cellcolor{myblue!18}0.9425 & \cellcolor{myblue!28}0.6543 & \cellcolor{myblue!35}0.6163 & \cellcolor{mypink!7}0.0327 & \cellcolor{mypink!1}0.6751 & \cellcolor{myblue!32}0.6419 & \cellcolor{myblue!0}0.6425 & \cellcolor{mypink!15}0.0334 & \cellcolor{myblue!16}0.6544 & \cellcolor{myblue!0}0.6451 \\
    \midrule[1.2pt]
    \multicolumn{13}{c}{Llama-3.1-8B-Instruct}\\
    \midrule[1.2pt]

    Vanilla & 0.0674 & 0.0645& 1.0 & 0.6176 & 0.0697 & 0.0741 & 1.0 & 0.6322 & 0.0645 & 0.0650 & 1.0 & 0.6461\\
    fully-finetuned  & 0.9247  & 0.9767 & 0.0 & 0.6276 & 0.9238 & 0.9719 & 0.0 & 0.6276 & 0.9463 & 0.9789 & 0.0 & 0.6276\\
    \midrule[0.2pt]

    GradDiff & 0.3072 & 0.0764 & 0.5497 & 0.5481 & 0.2897 & 0.0327 & 0.5666 & 0.5890 & 0.3098 & 0.0325 & 0.5638 & 0.5713\\
    GradDiff w. GU  & \cellcolor{myblue!40}0.3449 & \cellcolor{mypink!84}0.0756 & \cellcolor{mypink!1}0.5475 & \cellcolor{myblue!62}0.5659 & \cellcolor{myblue!100}0.4639 & \cellcolor{mypink!7} 0.0327 & \cellcolor{myblue!45}0.6402 & \cellcolor{myblue!43}0.6276 & \cellcolor{myblue!40}0.3408 & \cellcolor{mypink!7}0.0325 & \cellcolor{myblue!34}0.6233 & \cellcolor{myblue!45}0.5771 \\
    CEU & 0.1500 & 0.0291 & 0.5311 & 0.5465 & 0.0348 & 0.0327 & 0.9180 & 0.0000 & 0.0348 & 0.0325 & 0.8689 & 0.0000\\
    CEU w. GU  & \cellcolor{myblue!100}0.3050 & \cellcolor{myblue!29}0.0291 & \cellcolor{mypink!1}0.5091 & \cellcolor{myblue!100}0.6144 & \cellcolor{myblue!100}0.4470 & \cellcolor{mypink!7} 0.0327 & \cellcolor{mypink!0}0.6370 &\cellcolor{myblue!100} 0.6411 & \cellcolor{myblue!100}0.6987 & \cellcolor{mypink!7}0.0325 & \cellcolor{mypink!0}0.6856 & \cellcolor{myblue!100}0.6773 \\
    DPO & 0.5852 & 0.2854 & 0.5593 & 0.5774 & 0.4902 & 0.2476 & 0.4607 & 0.2390 & 0.7038 & 0.3366 & 0.4451 & 0.3281\\
    DPO w. GU  & \cellcolor{myblue!30}0.5840 & \cellcolor{mypink!60}0.2445 & \cellcolor{myblue!25}0.5702 & \cellcolor{myblue!37}0.5813 & \cellcolor{myblue!80}0.5512 & \cellcolor{mypink!14}0.2329 & \cellcolor{myblue!24}0.4661 & \cellcolor{myblue!100}0.3523 & \cellcolor{myblue!60}0.7557 & \cellcolor{mypink!14}0.3188 & \cellcolor{myblue!24}0.4565 & \cellcolor{myblue!100}0.4481 \\

    NPO & 0.3861 & 0.0811 & 0.7948 & 0.5717 & 0.2071 & 0.0648 & 0.7440 & 0.5839 & 0.2435 & 0.0684 & 0.7516 & 0.6104\\
    NPO w. GU  & \cellcolor{myblue!56}0.4006 & \cellcolor{mypink!28}0.0818 & \cellcolor{mypink!0}0.7972 & \cellcolor{myblue!27}0.5842 & \cellcolor{myblue!80}0.2642 & \cellcolor{myblue!1}0.0664 & \cellcolor{myblue!16}0.7447 & \cellcolor{myblue!69}0.6173 & \cellcolor{myblue!100}0.3476 & \cellcolor{myblue!0}0.0708 & \cellcolor{mypink!0}0.7396 & \cellcolor{myblue!0}0.6095 \\

    SatImp & 0.9505 & 0.9037 & 0.5186 & 0.6269 & 0.7967 & 0.7391 & 0.3641 & 0.6013 & 0.6794 & 0.6436 & 0.3846 & 0.6265\\
    SatImp w. GU  & \cellcolor{myblue!1}0.9342 & \cellcolor{mypink!17}0.7251 & \cellcolor{myblue!23}0.5248 & \cellcolor{myblue!28}0.6326 & \cellcolor{myblue!56}0.8120 & \cellcolor{mypink!69}0.4872 & \cellcolor{myblue!26}0.3782 & \cellcolor{myblue!30}0.6163 & \cellcolor{myblue!42}0.6978 & \cellcolor{mypink!100}0.4360 & \cellcolor{myblue!22}0.3891 & \cellcolor{myblue!0}0.6118 \\

    SimNPO & 0.8256 & 0.3101 & 0.6178 & 0.6270 & 0.7814 & 0.2529 & 0.4762 & 0.6040 & 0.6530 & 0.2110 & 0.4789 & 0.6029\\
    SimNPO w. GU  & \cellcolor{myblue!16}0.8284 & \cellcolor{mypink!89}0.1177 & \cellcolor{myblue!70}0.7810 & \cellcolor{myblue!2}0.6269 & \cellcolor{myblue!58}0.8067 & \cellcolor{mypink!100}0.1425 & \cellcolor{myblue!100}0.7424 & \cellcolor{myblue!31}0.6227 & \cellcolor{myblue!83}0.7103 & \cellcolor{mypink!65}0.1140 & \cellcolor{myblue!100}0.7027 & \cellcolor{myblue!35}0.6519 \\

    UNDIAL & 0.5679 & 0.0683 & 0.7964 & 0.7089 & 0.5453 & 0.0506 & 0.5812 & 0.6781 & 0.6213 & 0.0495 & 0.5364 & 0.6934\\
    UNDIAL w. GU  & \cellcolor{myblue!100}0.7520 & \cellcolor{mypink!24}0.0671 & \cellcolor{myblue!21}0.8115 & \cellcolor{myblue!57}0.7257 & \cellcolor{myblue!100}0.9191 & \cellcolor{mypink!30}0.0477 & \cellcolor{myblue!26}0.6029 & \cellcolor{myblue!47}0.6860 & \cellcolor{myblue!100}0.9349 & \cellcolor{mypink!24}0.0468 & \cellcolor{myblue!22}0.5361 & \cellcolor{myblue!0}0.6802 \\

    WGA & 0.7644 & 0.0745 & 0.7161 & 0.6258 & 0.7299 & 0.0339 & 0.6199 & 0.6256 & 0.6560 & 0.0339 & 0.6552 & 0.6331\\
    WGA w. GU  & \cellcolor{myblue!44}0.7915 & \cellcolor{mypink!12}0.0785 & \cellcolor{myblue!18}0.7513 & \cellcolor{myblue!28}0.6297 & \cellcolor{myblue!45}0.7577 & \cellcolor{mypink!22}0.0327 & \cellcolor{mypink!0}0.6184 & \cellcolor{myblue!32}0.6425 & \cellcolor{myblue!91}0.7122 & \cellcolor{mypink!7}0.0325 & \cellcolor{mypink!6}0.6225 & \cellcolor{myblue!0}0.6024 \\

    \bottomrule[1.5pt]
    \end{tabular}
    }
\end{table*}

%% file: table/muse.tex
\begin{table*}[!htbp]
    \centering
    \caption{MUSE benchmark results for Llama-2-7b-hf on Books and News splits and WMDP benchmark results for zephyr-7b-beta on cyber split. $\downarrow$ indicates smaller values are better, while $\uparrow$ indicates larger values are better. 
    }
    \vspace{-0.5em}
    \label{tab:muse_benchmark}
    \resizebox{\textwidth}{!}{
    \begin{tabular}{lcccccccc}
        \toprule[1.5pt]
        \multirow{2}{*}{\textbf{Method}} & \multicolumn{3}{c}{\textbf{MUSE Books}} & \multicolumn{3}{c}{\textbf{ MUSE News}} & \multicolumn{2}{c}{\textbf{ WMDP cyber}}\\
        \cmidrule(lr){2-4} \cmidrule(lr){5-7} \cmidrule(lr){8-9}
        & \begin{tabular}[c]{@{}c@{}}ES Un. $\downarrow$\end{tabular} & 
        \begin{tabular}[c]{@{}c@{}}Priv. Leak. $\rightarrow 0$\end{tabular} & 
        \begin{tabular}[c]{@{}c@{}}ROUGE Re. $\uparrow$\end{tabular} & 
        \begin{tabular}[c]{@{}c@{}}ES Un. $\downarrow$\end{tabular} & 
        \begin{tabular}[c]{@{}c@{}}Priv.  Leak. $\rightarrow 0$\end{tabular} & 
        \begin{tabular}[c]{@{}c@{}}ROUGE Re. $\uparrow$\end{tabular} &
        \begin{tabular}[c]{@{}c@{}}Un.  acc. $\downarrow$\end{tabular} & 
        \begin{tabular}[c]{@{}c@{}}mmlu   acc. $\uparrow$\end{tabular} 
        \\
        \midrule[1.5pt]
        \multicolumn{1}{c}{\textbf{}} & \multicolumn{6}{c}{\textbf{Llama-2-7b-hf}} & \multicolumn{2}{c}{\textbf{zephyr-7b-beta}} \\
        \midrule[1.2pt]
        Vanilla    &  0.01 & 8.16 & 0.68       & 0.02 & -4.72  & 0.56  & 0.4453 & 0.5845 \\
        fully-finetuned    &  0.92 & -57.34 & 0.69        & 0.29 & -99.81  & 0.55  & - & - \\
        \midrule[0.2pt]
        GD         & 0.0079 & -24.5562 & 0.0        & 0.0116 & 88.2242  & 0.3971  & 0.2420 & 0.4772  \\
        GD w. GU         & 0.0079 & -24.6394 & 0.0       & \cellcolor{mypink!30}0.0085 & \cellcolor{myblue!12}88.0562  & \cellcolor{myblue!22}0.3992 & \cellcolor{mypink!32}0.2375 & \cellcolor{myblue!22}0.4937  \\
        CEU        & 0.0079 & -58.8018 & 0.0        & 0.0079 & -7.3468  & 0.0 &  0.2455 &0.2689      \\
        CEU w. GU        & 0.0079 & \cellcolor{myblue!12}-58.0251 & 0.0        & 0.0182 & 66.1418  & \cellcolor{myblue!100}0.4349 & 0.2455 & 0.2689        \\
        NPO        & 0.3933 & -54.4933 & 0.6185   & 0.1021 & -85.8312 & 0.5050 & 0.3457 & 0.5422  \\
        NPO w. GU & \cellcolor{mypink!12}0.3822    & \cellcolor{myblue!22}-53.7352 & \cellcolor{myblue!22} 0.6251   & 0.1175 & -86.04 & 0.5037  & 0.3668 & \cellcolor{myblue!22}0.5518   \\
        SatImp     & 0.7710 & -58.3950 & 0.6114   & 0.2287 & -99.8741 & 0.3991 & 0.4177 &0.5654  \\
        SatImp w. GU    & \cellcolor{mypink!32}0.7321 & \cellcolor{myblue!12}-57.3851 & \cellcolor{myblue!36} 0.6310   & \cellcolor{mypink!22}0.1943 & -99.8740 &\cellcolor{myblue!26} 0.4100 & \cellcolor{mypink!12}0.4157 & \cellcolor{myblue!12}0.5674  \\
        SimNPO     & 0.1407 & -54.2530 & 0.5103   & 0.1778 & -99.8741 & 0.4114  & 0.4192 & 0.5658 \\
        SimNPO w. GU    & \cellcolor{mypink!62}0.0813 & \cellcolor{myblue!42}-46.4866 & \cellcolor{myblue!69}0.5980   & \cellcolor{mypink!62}0.0957 & -99.8740 & \cellcolor{myblue!2} 0.4143 & \cellcolor{mypink!12}0.4177 &0.5663   \\

        UNDIAL     & 0.0231 & -18.3432 & 0.6309   & 0.0110 & -98.9085 & 0.1928  & 0.3829 & 0.5596  \\
        UNDIAL w. GU    & \cellcolor{mypink!12}0.0219 &\cellcolor{myblue!12} -18.2137 & \cellcolor{myblue!12}0.6370   &  0.0168 & -99.37 & \cellcolor{myblue!100}0.3638  & \cellcolor{mypink!12}0.3789 & 0.5612\\
        WGA        & 0.0079 & -49.9445 & 0.4689   & 0.0102 & 101.1335 & 0.4602 & 0.2455 & 0.2550   \\
        WGA w. GU       & 0.0079 &\cellcolor{myblue!58} -40.1072 & 0.4682   & \cellcolor{mypink!35}0.0084 & 108.14 & \cellcolor{myblue!7} 0.4615 & 0.3819 & \cellcolor{myblue!100}0.5498  \\
        \bottomrule[1.5pt]
        \vspace{-3em}
    \end{tabular}
    }
\end{table*}

%% file: contents/5.related_work.tex
\vspace{-1em}
\section{Related Work}
\vspace{-0.5em}
\paragraph{LLM Unlearning}
Machine unlearning for LLMs aims to remove the influence of a designated forget set while preserving capabilities on the retain distribution, ideally approximating retraining on retained data~\citep{openunlearning}. Existing approaches largely operate by modifying post-training updates~\citep{dong-etal-2025-undial,ji2024reversing}.
\emph{Gradient-based unlearning} performs loss maximization on forget samples (often combined with retain regularization), but can be unstable and prone to catastrophic forgetting or collapse~\citep{thudi2022unrolling, izzo2021approximate}. \emph{Reweighting / preference-style objectives} such as NPO/SimNPO/DPO improve optimization stability by reshaping the effective gradient contributions of forget samples~\citep{NPO_zhang2024negative,fan2024simplicity_simnpo,rafailov2023dpo}. Other lines include \emph{representation engineering} (e.g., manipulating hidden states or distilling away targeted behaviors)~\citep{wang2025llm, ginart2019making} and \emph{inference-time masking} that corrupts prompts or embeddings to suppress recall without changing parameters~\citep{le2025survey, liu2024large}. Despite progress, many methods exhibit \emph{non-robust} unlearning: traces of the forgotten content can remain recoverable under adversarial prompting or extraction evaluations, and the unlearned model can remain distinguishable from a retrained reference~\citep{maini2024tofu, li2024wmdp, openunlearning}. These failure modes reflect a central structural challenge: LLM knowledge is stored in distributed, overlapping representations, so updates intended to erase a concept often interfere with unrelated behaviors or fail to fully remove the targeted mechanism~\citep{maini2024tofu, microedit2024, ghosal2025memorization, liu-etal-2025-disentangling}. 

\vspace{-1.5em}
\paragraph{Geometric Control via Orthogonal Projection}
A natural response to interference is to impose \emph{geometric constraints} on parameter updates. In continual learning, gradient projection/surgery methods preserve prior tasks by restricting new updates to directions orthogonal to protected gradients or subspaces~\citep{kirkpatrick2017overcoming, NEURIPS2020_3fe78a8a, farajtabar2020orthogonal}. Recent unlearning and editing work adapts this principle by constructing subspaces from activations or gradient statistics and projecting updates to confine forgetting~\citep{fang2025alphaedit, feng-etal-2025-geoedit,kim2024negmerge, he2025TowardsNatural, cadet2024deep}. 
Several recent unlearning methods explicitly reduce forget-retain interference by enforcing hard or soft orthogonality between forget updates and a retain subspace, via per-sample gradient projection (OrthoGrad~\cite{shamsian2025OrthoGrad}), penalty-based cosine regularization (UNO~\cite{mandal2025uno}), constrained min-norm updates~\cite{block2025machine}, or representation-level subspace projection for incremental unlearning (FG-OrIU~\cite{feng2025FG-OrIU}). PGU~\citep{hoang2024learn}, UNSC~\citep{chen2024unsc}, and SEMU~\citep{sendera2025semu} impose projection or null-space constraints derived from covariance- or SVD-style approximations.   However, these methods rely on coarse estimates (e.g., small-batch SVD/covariance) and typically define orthogonality under a \emph{fixed Euclidean geometry},
in contrast, GU derives orthogonality from a \emph{retain-invariance} specification and implements it as an \textbf{optimizer-metric} $H$ projection layer that is objective-preserving, pipeline-agnostic, and compatible with standard preconditioned optimizers. This motivates the need for \emph{optimizer-aware} orthogonalization that scales to LLMs and offers principled control of the forget-retain trade-off. (see Appendix~\ref{sec:discussion_rws} for detailed comparisons).

%% file: contents/6.conclusion.tex
\vspace{-4.5mm}
\section{Conclusion}
\label{sec:conclusion}
\vspace{-2mm}
We formalize theoretically sound LLM unlearning as enforcing retain-invariance, showing its equivalence to forgetting updates orthogonal to retain-gradient subspace under an SPD metric. We propose GU, a plug-and-play projection that attaches to existing gradient-based unlearning and improves forget-retain trade-off acrossTOFU, MUSE, WMDP.
\vspace{-10mm}
\paragraph{Limitations \& Future Work}
Our method relies on the availability of a representative retain dataset to construct the orthogonal subspace. In scenarios where retain data is unavailable, constructing an accurate $T_r$ becomes challenging. Additionally, our implementation approximates the Riemannian metric $H$ using diagonal information from optimizers to ensure computational feasibility on LLMs. While efficient, this ignores higher-order parameter correlations that could theoretically offer finer-grained control.

\section*{Broader Impact}

This work advances machine unlearning toward precise, targeted erasure of specific training influences—e.g., private attributes, copyrighted passages, or hazardous know-how—from large language models (LLMs). Reliable targeted removal can materially improve real-world deployability of foundation models by enabling post hoc compliance with deletion obligations and reducing downstream misuse risk, especially for widely distributed or open-weight models where “recall” is not feasible.

A central challenge is that effective forgetting frequently causes collateral damage—utility loss on unrelated retained knowledge—making unlearning brittle and difficult to audit. We address this by formalizing a concrete, testable notion of retain-invariance: under a local (first-order) model of training dynamics, retain performance is preserved when updates are orthogonal to the retain-gradient subspace under the optimizer’s geometry. This geometric view turns “disentanglement” from a vague heuristic into an explicit constraint with a transparent operational meaning: it isolates and removes the update component most responsible for degrading retained behavior. By making the intervention more structured and predictable, our approach can improve controllability and reduce unintended model deterioration relative to unconstrained forgetting updates.

We also recognize dual-use concerns. In principle, any unlearning tool could be repurposed to remove safety alignment behaviors or to selectively weaken safeguards. Our contribution does not eliminate this risk, but it does promote auditability and boundedness: the update is explicitly decomposed into retain-tangential versus retain-orthogonal components, making it easier to inspect what is being changed and to subject the result to verification. Our experiments rely exclusively on established public benchmarks, and we recommend that any deployment be paired with rigorous verification protocols (e.g., utility regression suites, privacy leakage and extraction stress tests, and change auditing across safety-critical behaviors) to ensure that unlearning is used in accordance with data-minimization principles and applicable legal standards.

%% file: contents/appendix.tex
\appendix
\onecolumn

\textbf{\Large Appendix}

We include extended discussion~\ref{sec:discussion}, complete algorithm~\ref{sec:alg}/proof~\ref{sec:proof} details, additional experiments~\ref{sec:extra_experimrnts} and ablation studies~\ref{app:ablation-cost} (with ablations~\ref{app:opt-ablation} and cost analysis~\ref{app:cost}), and qualitative results~\ref{sec:qualitative}.

\input{contents/a.discussion}
\input{contents/a.algorithm}
\input{contents/a.proofs}

\input{contents/a.exp}

\input{contents/a.qualititive}

%% file: contents/a.discussion.tex
\section{Discussion}
\label{sec:discussion}
\subsection{Discussion on Optimizer-induced Metric}
\label{sec:discussion_adam}
We work in parameter space $\mathbb{R}^p$.
Let $H\succ0$ denote the optimizer-induced metric; for Adam, $H=W^\top W$ with
\begin{align}
\label{eq:whiten}
W = \mathrm{diag}\big(1/\sqrt{\hat v+\varepsilon}\big),
\end{align}
where $\hat v$ is Adam's second-moment accumulator.
For an optimizer with a (possibly time-varying) linear preconditioner $P_t$ such that the step direction is $d_t = -P_t g_t$,
set the metric to
\[
H_t := P_t^\top P_t,
\]
and use $H_t$ consistently to build the retain basis $U$ and the projectors $P_{T_r}^{(H_t)},P_\perp^{(H_t)}$.
Then the local retain invariance $\Delta^{(1)}L_r=\langle \nabla L_r,\Delta\theta\rangle_{H_t}=0$ is equivalent to $\Delta\theta\in T_r^\perp$ under $H_t$,
and all first-order safety statements carry the same.

\subsection{Constructing the retain-orthogonal space.}
\label{sec:discussion_retain}
To protect retained behavior, we here introduce how we derive $T_r$ on $D_r$ from a retain loss:
\[
L_{r}(\theta)\;=\;\mathbb{E}_{x\in D_r}\Big[\mathrm{KL}\big(\pi_\theta(\cdot\mid x)\,\|\,\pi_{\mathrm{ref}}(\cdot\mid x)\big)\Big],
\]
which yields low-variance and stable gradients, a zero-gradient baseline near $\pi_\theta\approx\pi_{\mathrm{ref}}$, and alignment with preserving output style.

From a small retain mini-batch $B_r\subset D_r$ on selected tensors, we collect retain gradients to form $U$ and orthonormalize in whitened coordinates (Gram–Schmidt) so that $U^\top H U=I_k$.
For any $v\in\mathbb{R}^p$, the projection $P^{(H)}_{\perp}v$ lies in $T_r^\perp$ and satisfies $U^\top H\,P^{(H)}_{\perp}v=0$.
Hence its $H$ inner products with all retain-tangential directions vanish.
Equivalently, $P^{(H)}_{\perp}$ removes the tangential component along $T_r$ and preserves only the $H$ normal component, thereby eliminating retain–forget entanglement to first order while keeping the component that drives forgetting.



\subsection{{Discussion on Projection-based Unlearning Methods}}
\label{sec:discussion_rws}

This section provides a detailed comparison between GU and prior projection-based unlearning methods (UNSC~\citep{chen2024unsc}, PGU~\citep{hoang2024learn}, SEMU~\citep{sendera2025semu}, OrthoGrad~\cite{shamsian2025OrthoGrad}, UNO~\cite{mandal2025uno}, FG-OrIU~\cite{feng2025FG-OrIU}, and MinNorm-OG~\cite{block2025machine}), and clarifies why we do not include them as baselines in our LLM experiments. While these works share the high-level theme of ``projection'', they are developed under different task assumptions and operate on different geometric objects, which makes a faithful LLM-scale implementation either mismatched to our setting or computationally prohibitive.

\paragraph{The Target Unlearning Regime is Fundamentally Different.}
UNSC~\citep{chen2024unsc} and PGU~\citep{hoang2024learn} are designed for supervised image classification with a small, fixed label space, where one can meaningfully define \emph{class-conditional} representation statistics and compute dataset-level matrices (e.g., per-class covariances or Gram matrices) to protect retain classes. SEMU~\citep{sendera2025semu} similarly assumes a supervised setup and constructs low-rank subspaces from forget-set gradients/statistics to parameterize the unlearning update. In contrast, our target is large-scale \emph{LLM unlearning in the post-training regime} (offline SFT and preference tuning), where outputs are open-vocabulary sequences, ``classes'' are not well-defined, and unlearning is performed \emph{post hoc} as a stage in a multi-step alignment pipeline. Under this regime, the structural assumptions behind class-based or dataset-statistics-based projections (and the need to repeatedly recompute such statistics across stages) do not directly apply.

\paragraph{These Methods Project onto Different \emph{Geometric Objects}.}
Even within ``projection'' methods, the protected object varies substantially.
(1) \textbf{Representation/statistics-based subspaces.}
UNSC and PGU define protected directions through proxies derived from activations or dataset statistics (e.g., class-conditional activation subspaces, retain-only Gram/covariance structure), with the goal that preserving these proxies approximately preserves retain behavior.
FG-OrIU~\cite{feng2025FG-OrIU} further moves projection to the representation level for incremental unlearning by estimating layerwise forget/retain feature subspaces (typically via SVD) and applying forward-feature and backward-gradient projections while training low-rank adapters.
(2) \textbf{Forget-dominated low-rank parameterizations.}
SEMU constructs a low-rank subspace primarily from $D_f$-derived information and restricts updates to that subspace; retention is handled indirectly through low rank and step-size control rather than through an explicit retain-side invariance constraint.
(3) \textbf{Gradient-orthogonality constraints.}
OrthoGrad~\cite{shamsian2025OrthoGrad} enforces hard orthogonalization by estimating a retain subspace from \emph{per-sample} retain gradients (via QR) and projecting the forget update onto its orthogonal complement. In addition, OrthoGrad orthogonalizes gradients in the \emph{Euclidean geometry}, whereas GU orthogonalizes updates in the optimizer-induced SPD geometry derived from a retain-invariance specification; the latter is aligned with preconditioned optimization and therefore controls retain drift under the actual LLM training dynamics.
UNO~\cite{mandal2025uno} (for generative models) instead imposes a \emph{soft} cosine-orthogonality regularizer between forget/retain gradients inside the objective, making the effect sensitive to penalty tuning.
MinNorm-OG~\cite{block2025machine} formulates unlearning as constrained optimization with an explicit feasibility constraint $\Delta \in \mathrm{span}(g_{\text{retain}})^\perp$, selecting approximately minimum-norm updates under this constraint.

GU differs in the \emph{definition} of the protected geometry: rather than using indirect proxies (activations/Grams) or relying on forget-only low-rank parameterizations, we define the retain geometry directly in parameter space as the retain-gradient span under the \textbf{optimizer-induced SPD metric} $H$,
\[
T_r(\theta) \;=\; \mathrm{span}\{\nabla_\theta \ell_r(x_r;\theta): x_r \in D_r\},
\]
and we prove that \emph{local retain invariance} is equivalent to $H$-orthogonality to $T_r(\theta)$. Within the retain-orthogonal set, the GU update is the steepest-descent direction for the chosen unlearning objective. Thus, in GU the retain-gradient subspace is the primary geometric object and is explicitly aligned with the optimizer geometry, rather than a representation-level proxy constructed from activations, Grams, or ad hoc low-rank parameterizations.

\paragraph{Scalability is the main practical barrier for LLM baselines.}
A second key distinction is computational scaling.
UNSC/PGU require computing and factorizing per-class activation covariances or retain Grams (often layerwise), and FG-OrIU/SEMU require repeated layerwise SVD-style operations over high-dimensional representations or gradient-derived matrices. For LLMs with hidden dimension $d\approx 4\text{k}$--$8\text{k}$ and many layers, these procedures involve $d\times d$ objects across layers, leading to prohibitive memory and factorization cost, especially when unlearning must be repeated across stages on top of existing SFT/preference tuning. Any ``lightweight'' adaptation (e.g., collapsing outputs into pseudo-classes, dropping most layers, or aggressively subsampling statistics) would deviate substantially from the original algorithms, making negative results difficult to interpret.

GU is explicitly designed to avoid $d^2$-scale computations: we never form activation/Gram matrices and never run layerwise SVDs. Instead, we maintain a low-rank retain basis $B\in\mathbb{R}^{d\times k}$ in parameter space using streaming updates from retain gradients, and project forget gradients via vector-level operations under $H$ (e.g., $g_f^\perp = g_f - \Pi_{T_r}^{H} g_f$), yielding $O(kd)$ cost per step (plus small $O(k^2)$ bookkeeping). This turns geometric disentanglement into a lightweight \emph{routing layer} that introduces only modest overhead on top of existing LLM unlearning pipelines (e.g., SimNPO, SatImp, WGA, NPO), making it practically deployable at LLM scale.

\paragraph{Why we do not include these baselines in our LLM experiments.}
For UNSC/PGU/FG-OrIU/SEMU, a faithful implementation at LLM scale is either infeasible (due to repeated $d^2$-scale statistics/factorizations) or would require non-faithful approximations that change the method qualitatively. For OrthoGrad, the need for per-sample retain gradients and QR-style subspace estimation makes scaling to LLMs expensive in both memory and runtime, and the method is not evaluated in LLM unlearning benchmarks. UNO targets generative-model unlearning and optimizes a different objective class. MinNorm-OG studies constrained solutions with minimum-norm selection under orthogonality constraints, emphasizing global optimality/complexity control rather than plug-and-play integration into multi-stage LLM alignment pipelines.

Therefore, instead of including potentially non-faithful approximations as LLM baselines, we compare against strong LLM-suitable unlearning methods and conduct ablations that directly probe GU's geometry (metric choice, subspace rank), which more cleanly answers our central question: whether optimizer-aware geometric disentanglement improves the Pareto trade-off between forgetting and retention in the LLM post-training regime.

%% file: contents/a.algorithm.tex
\section{Algorithm Details}
\label{sec:alg}
Here, we introduce the detailed practical implementation of Geometric-Disentanglement Unlearning (GU) in Algorithm~\ref{alg:GU}.
At each step we only project a small, automatically selected subset of trainable tensors: the last $K$ Transformer blocks (largest layer indices detected from names such as \texttt{.layers.$i$.}, \texttt{.h.$i$.}, \texttt{.blocks.$i$.}, \texttt{.decoder.layers.$i$.)}), plus the final normalization(s) and the output head. This keeps cost and memory small while targeting the most forget-sensitive layers.

We work in coordinates whitened by the Adam preconditioner. For each selected parameter tensor $p$, let $v_p$ denote Adam’s second moment estimate; then

\begin{equation*}
    W_p  =  \mathrm{diag}\!\Big(\tfrac{1}{\sqrt{v_p+\varepsilon}}\Big),
\qquad
H_p  =  W_p^{\!\top}W_p
\quad(\text{diagonal}).
\end{equation*}

In practice, we \emph{bind} to the optimizer state and reuse $v_p$ without extra memory; if unavailable, we maintain an EMA of squared gradients. All projections are performed in whitened coordinates $\tilde g = W_p g$. Thus, $H$ is approximated by the Adam diagonal Fisher/Gauss–Newton surrogate already maintained during training.

\begin{algorithm}[h]
    \caption{Geometric-disentanglement Unlearning}
    \label{alg:GU}
        {
    \begin{algorithmic}[1]
    \REQUIRE At step $t$, Parameters $\theta_t$, optimizer-induced metric $H_t$ (whitener $W_t$), reference model $\theta_{\mathrm{ref}}$, batches $B_f,B_r$, weights $\gamma,\alpha$, \emph{any} forget loss function $L_f$ and retain loss function $L_r$
    \STATE Compute $L_f(\theta_t;B_f)$ and  $L_r(\theta_t;B_r,\theta_{\mathrm{ref}})$, form $L_{\mathrm{tot}}=\gamma L_f+\alpha L_r$ and total gradient $g_{\mathrm{tot}}=\nabla_\theta L_{\mathrm{tot}}$
    \STATE Compute retain KL anchor $L_r^{\mathrm{KL}}$ and gradient $g_r$, whiten $\tilde g_r=W_t g_r$ and update the retain basis $U_t$ to approximate $T_r(\theta_t)$.
    \STATE Recover forget gradient via $g_f=(g_{\mathrm{tot}}-\alpha g_r)/\gamma$ and whiten $\tilde g_f=W_t g_f$
    \STATE Decompose $\tilde g_f$ w.r.t.\ $U_t$ under $H_t$: obtain retain-orthogonal $\tilde g_f^\perp$ and a sign-selective, norm-capped tangent part $\tilde g_f^{\mathrm{tan,keep}}$; similarly get $\tilde g_r^{\mathrm{tan}}$ 
    \STATE Form whitened GU direction $\tilde g_{\mathrm{GU}}=\gamma(\tilde g_f^\perp+\tilde g_f^{\mathrm{tan,keep}})+\alpha\,\tilde g_r^{\mathrm{tan}}$, map back $g_{\mathrm{GU}}=W_t^{-1}\tilde g_{\mathrm{GU}}$, and let the base optimizer step with gradient $g_{\mathrm{GU}}$ to obtain $\theta_{t+1}$
    \ENSURE Updated parameters $\theta_{t+1}$
    \end{algorithmic}}
\end{algorithm}

\paragraph{How we find $P^{(H)}_{T_r}$ and $P^{(H)}_{\perp}$.}
Per selected tensor $p$, we maintain a small basis $U_p=\{u_{p,1},\ldots,u_{p,m}\}$ (with $m\le k$) that spans the retain tangent subspace $T_r$ under $H_p$. The basis is refreshed at low frequency using one backward pass of a lightweight retain anchor
\[
\mathcal{L}_r(\theta) = \mathbb{E}_{x\in D_r}\mathrm{KL}\!\big(\pi_\theta(\cdot|x)\,\Vert\,\pi_{\mathrm{ref}}(\cdot|x)\big).
\]
For each $p$, we compute $g_{r,p}=\partial\mathcal{L}_r/\partial p$, whiten $\tilde g_{r,p}=W_p g_{r,p}$, and run Gram–Schmidt against the current $U_p$ in float32; if the relative residual $\|\mathrm{res}\|/\|g_{r,p}\|$ exceeds \texttt{residual\_keep\_thresh} and $|U_p|<k$, we append the normalized residual (stored in fp16). In whitened coordinates the projectors are
\begin{align*}
    P^{(H)}_{T_r}(p) &= U_p\,(U_p^{\!\top}U_p)^{-1}U_p^{\!\top},\\
    P^{(H)}_{\perp}(p) &= I-P^{(H)}_{T_r}(p),
\end{align*}
implemented via “accumulate/subtract along $U_p$” formulas.

\paragraph{Projected update used in training.}
Let the training objective be $\mathcal{L}=\gamma\,\mathcal{L}_f+\alpha\,\mathcal{L}_r$, where $\mathcal{L}_f$ is any forget loss (DPO, NPO, UNDIAL, SimNPO, CEU, WGA, SaT-IMP, or plain NLL). After the normal backward pass, $g_{\mathrm{tot},p}=\gamma g_{f,p}+\alpha g_{r,p}$ is stored in \texttt{p.grad}. Right before the optimizer step we:
\textit{(i)} recompute the scalar retain anchor once to obtain $g_{r,p}$ (this also refreshes $U_p$ when scheduled);
\textit{(ii)} recover $g_{f,p}=(g_{\mathrm{tot},p}-\alpha g_{r,p})/\gamma$ without an extra forward pass;
\textit{(iii)} whiten: $\tilde g_{f,p}=W_p g_{f,p}$, $\tilde g_{r,p}=W_p g_{r,p}$;
\textit{(iv)} project:
\[
\tilde g^{\,\text{safe}}_{f,p}=P^{(H)}_{\perp}(p)\,\tilde g_{f,p}, 
\qquad 
\tilde g^{\,\text{tan}}_{r,p}=P^{(H)}_{T_r}(p)\,\tilde g_{r,p},
\]
optionally adding a \emph{sign-aware, capped} tangential component from $\tilde g_{f,p}$: for each basis vector $u_{p,j}$, keep $a_j u_{p,j}$ with $a_j=\langle \tilde g_{f,p},u_{p,j}\rangle$ only if $a_j\,b_j<-\tau$ where $b_j=\langle \tilde g_{r,p},u_{p,j}\rangle$, then cap the resulting tangential norm (see trust region);
\textit{(v)} de-whiten and overwrite the final gradient:
\[
g^{\star}_p  =  \gamma\,W_p^{-1}\tilde g^{\,\text{safe}}_{f,p} + \alpha\,W_p^{-1}\tilde g^{\,\text{tan}}_{r,p}.
\]
The base optimizer (Adam) takes the step with its usual learning rate.

\paragraph{Trust region in practice.}
We do not run a separate line-search; instead we enforce an \emph{anisotropic trust region in whitened coordinates} that limits motion along $T_r$ relative to the retain-orthogonal direction:
\[
\big\| \tilde g^{\,\text{tan,keep}}_{f,p} \big\|_2
 \le  \kappa \,\big\| \tilde g^{\,\text{safe}}_{f,p} \big\|_2 ,
\qquad 0\le \kappa \le 1,
\]
with $\kappa=$ (default $0.5$). We also use a sign threshold $\tau=$ (default $0$) and only keep tangential components where forget and retain gradients have \emph{opposite} signs along the same basis vector ($\langle \tilde g_{f,p},u_{p,j}\rangle\cdot\langle \tilde g_{r,p},u_{p,j}\rangle<-\tau$). The pair $(\kappa,\tau)$ acts as a stable trust-region controller; the global step size remains the optimizer's learning rate. We list $\kappa$ and $\tau$ in the hyperparameter table of the appendix.

\subsection{Sign-Aware Selective Projection}
\label{sec:discussion_sign}
Not all retain-tangential components of $\nabla L_f$ are harmful to $L_r$.
In whitened coordinates $\tilde g_f:=W\nabla L_f$, $\tilde g_r:=W\nabla L_r$, let
$a_i=\langle \tilde g_f,u_i\rangle$, $b_i=\langle \tilde g_r,u_i\rangle$.
We keep only the \emph{opposite-signed} retain-tangential components of $\tilde g_f$ (which locally \emph{decrease} $L_r$), and discard same-signed (harmful) ones.
We also cap their magnitude to avoid drifting within $T_r$:
\begin{align}
\label{eq:sign-aware}
\text{Keep $u_i$ if } a_i b_i &< -\tau
\quad(\tau\ge 0), \notag \\
\big\|\textstyle\sum_{i: a_i b_i<-\tau} a_i u_i\big\|
&\le
\kappa \big\|P_{\perp}^{(H)}\tilde g_f\big\|
\quad(0<\kappa\le 1).
\end{align}
The resulting forget direction in whitened coordinates is
$\tilde g_f^{\rm sel}=P_{\perp}^{(H)}\tilde g_f + \text{tan\_keep}$, while the retain direction is
$\tilde g_r^{\rm nor}=P_{T_r}^{(H)}\tilde g_r$.
Mapping back with $W^{-1}$ yields the final gradient
$\nabla_\theta \leftarrow \gamma g_f^{\rm sel} + \alpha g_r^{\rm nor}$.
This preserves first-order safety (harmful tangential parts are removed) while not wasting helpful opposite-signed components.

\paragraph{Sign-aware refinement.}
With the sign-aware rule (Eq.~\eqref{eq:sign-aware}), the forget direction reads
$P^{(H)}_\perp g_f + \sum_{i\in\mathcal{K}} a_i u_i$ where $\mathcal{K}=\{i: a_i b_i<-\tau\}$.
Its contribution to the first-order retain change is
$-\rho\sum_{i\in\mathcal{K}} a_i b_i \le -\rho \tau \sum_{i\in\mathcal{K}} |a_i||b_i|\le0$,
so Proposition~\ref{prop:safety} strengthens to
\[
\Delta^{(1)} L_r \le -\rho\beta\|P^{(H)}_{T_r} g_r\|_H^2 -\rho \tau\!\!\sum_{i\in\mathcal{K}} |a_i||b_i| \le 0.
\]
The cap in Eq.~\eqref{eq:sign-aware} further bounds the tangential energy, preventing drift within $T_r$.

%% file: contents/a.proofs.tex
\section{Proofs}
\label{sec:proof}
\subsection{Proof of Proposition~\ref{prop:retain-subspace-orth}}
\label{sec:prooforth}

Formally, we reframe our Proposition~\ref{prop:retain-subspace-orth} as Prop~\ref{prop:retain-subspace-orth2}
\begin{proposition}[Retain gradient subspace and $H$-orthogonality]
\label{prop:retain-subspace-orth2}
Fix $\theta\in\mathbb{R}^p$ and an SPD matrix $H\succ0$ inducing the inner product
$\langle u,v\rangle_H := u^\top H v$ and norm $\|v\|_H := \sqrt{\langle v,v\rangle_H}$.
For each retain sample $x_r\in D_r$, assume $\ell_r(x_r;\theta)$ is differentiable and write its (Euclidean) gradient
$g(x_r):=\nabla_\theta \ell_r(x_r;\theta)\in\mathbb{R}^p$.
Define the retain gradient subspace and its $H$-orthogonal complement
\begin{align*}
    T_r(\theta)&:=\mathrm{span}\{ g(x_r): x_r\in D_r \}\subseteq\mathbb{R}^p,
    \\
    T_r(\theta)^\perp&:=\{ v\in\mathbb{R}^p:\ \langle v, g\rangle_H=0\ \forall g\in T_r(\theta) \}.
\end{align*}
For any finite (multi)set $S=\{x_{r_1},\dots,x_{r_m}\}\subset D_r$, define
\begin{align*}
    L_r^S(\theta):=&\sum_{i=1}^m \ell_r(x_{r_i};\theta),
    \\
    \nabla_\theta L_r^S(\theta)=&\sum_{i=1}^m g(x_{r_i})\in T_r(\theta).
\end{align*}
Then, for any direction $\Delta\theta\in\mathbb{R}^p$, the following are equivalent:
\begin{align*}
    \text{\emph{(i)}}\ \Delta\theta\in T_r(\theta)^\perp
\Longleftrightarrow 
\text{\emph{(ii)}}\ \langle \nabla_\theta L_r^S(\theta), \Delta\theta\rangle_H=0
\ \text{ for all finite } S\subset D_r.
\end{align*}
Equivalently, the first-order quantity $\Delta^{(1)}L_r^S(\theta;\Delta\theta):=\langle \nabla_\theta L_r^S(\theta),\Delta\theta\rangle_H$ vanishes for all finite $S$ iff $\Delta\theta\in T_r(\theta)^\perp$.
\end{proposition}

\begin{proof}
\textbf{Preliminaries.}
(i) $T_r(\theta)$ is a linear subspace of $\mathbb{R}^p$ by definition (finite linear combinations of the $g(x_r)$).
(ii) By linearity of the gradient,
$\nabla_\theta L_r^S(\theta)=\sum_{i=1}^m g(x_{r_i})\in T_r(\theta)$.

\smallskip
\noindent\emph{(i)$\Rightarrow$(ii).}
Assume $\Delta\theta\in T_r(\theta)^\perp$.
By definition of $T_r(\theta)^\perp$, $\langle \Delta\theta, v\rangle_H=0$ for all $v\in T_r(\theta)$.
In particular, since $\nabla_\theta L_r^S(\theta)\in T_r(\theta)$ for every finite $S$, we obtain
$\langle \nabla_\theta L_r^S(\theta), \Delta\theta\rangle_H=0$ for all finite $S\subset D_r$.

\smallskip
\noindent\emph{(ii)$\Rightarrow$(i).}
Assume $\langle \nabla_\theta L_r^S(\theta), \Delta\theta\rangle_H=0$ for all finite $S\subset D_r$.
Take a singleton set $S=\{x_r\}$.
Then $\nabla_\theta L_r^S(\theta)=g(x_r)$ and hence
$\langle g(x_r), \Delta\theta\rangle_H=0$ for every $x_r\in D_r$.
Let $v\in T_r(\theta)$ be arbitrary.
By definition of $T_r(\theta)$ there exist $x_{r_1},\dots,x_{r_m}$ and scalars $\alpha_1,\dots,\alpha_m$ with
$v=\sum_{i=1}^m \alpha_i g(x_{r_i})$.
Using bilinearity of $\langle\cdot,\cdot\rangle_H$ and the singleton orthogonality just shown,
\begin{align*}
    \langle v, \Delta\theta\rangle_H
&= \left\langle \sum_{i=1}^m \alpha_i g(x_{r_i}), \Delta\theta \right\rangle_H \\
&= \sum_{i=1}^m \alpha_i \langle g(x_{r_i}), \Delta\theta\rangle_H \\
&= 0.
\end{align*}
Since $v\in T_r(\theta)$ was arbitrary, $\Delta\theta\in T_r(\theta)^\perp$.

\smallskip
Combining the two directions yields the equivalence.
\end{proof}

\paragraph{Remark.}
If one prefers to identify $\Delta^{(1)}L_r^S$ with the true directional derivative, introduce the $H$-gradient $\nabla_\theta^H L := H^{-1}\nabla_\theta L$ so that
$D L_r^S(\theta)[\Delta\theta]=\langle \nabla_\theta^H L_r^S(\theta), \Delta\theta\rangle_H$.
The proof above is unchanged because $\mathrm{span}\{g(x_r)\}$ and $\mathrm{span}\{\nabla_\theta^H \ell_r(x_r;\theta)\}$ have the same $H$-orthogonal complement.

\subsection{Proof of Lemma~\ref{lem:steepest}}
\label{sec:prooflem:steepest}

\begin{proof}
Because $\mathcal C\subseteq T_r^\perp$, every $\Delta\theta\in\mathcal C$ satisfies
$\langle g_r,\Delta\theta\rangle_H=0$ (since $g_r\in T_r$).
Hence for $\Delta\theta\in\mathcal C$,
\[
\big\langle g_f+\alpha g_r,\ \Delta\theta\big\rangle_H
= \langle g_f,\Delta\theta\rangle_H .
\]
Thus minimizing the joint directional derivative over $\mathcal C$ is equivalent to minimizing
$\langle g_f,\Delta\theta\rangle_H$ over $\mathcal C$.

Now decompose $g_f=P_{T_r}^{(H)}g_f+P_\perp^{(H)}g_f$ with $H$-orthogonal components, and note that
$\Delta\theta\in T_r^\perp$ implies
$\langle P_{T_r}^{(H)}g_f,\Delta\theta\rangle_H=0$, so
\[
\langle g_f,\Delta\theta\rangle_H=\langle P_\perp^{(H)}g_f,\Delta\theta\rangle_H .
\]
By Cauchy--Schwarz,
$\langle P_\perp^{(H)}g_f,\Delta\theta\rangle_H\ge
- \|P_\perp^{(H)}g_f\|_H \|\Delta\theta\|_H \ge - \|P_\perp^{(H)}g_f\|_H$,
with equality achieved at
$\Delta\theta=- P_\perp^{(H)}g_f/\|P_\perp^{(H)}g_f\|_H$ when $P_\perp^{(H)}g_f\neq0$.
If $P_\perp^{(H)}g_f=0$, then $\langle g_f,\Delta\theta\rangle_H=0$ for all $\Delta\theta\in\mathcal C$,
so every feasible unit vector is optimal.
Therefore the stated $\Delta\theta_f^\star$ solves both problems over $\mathcal C$.
\end{proof}

\subsection{Proof of Proposition~\ref{prop:safety}}
\label{sec:proofsafety}

\begin{proof}
We proceed in three explicit steps.

First, we refer to the projector's properties.
By construction, $P_{T_r}^{(H)}$ is the $H$-orthogonal projector onto $T_r$ and $P_\perp^{(H)}:=I_p-P_{T_r}^{(H)}$
is the projector onto $T_r^\perp$. Both are $H$-self-adjoint and idempotent, and they are $H$-orthogonal in the sense that
$\langle P_{T_r}^{(H)} u,  P_\perp^{(H)} v\rangle_H=0$ for all $u,v$.
Moreover, since $g_r\in T_r$, we have $P_{T_r}^{(H)} g_r = g_r$ and $P_\perp^{(H)} g_r = 0$.

By definition of the $H$-gradient,
$\Delta^{(1)} L_r = \langle g_r,\Delta\theta\rangle_H$ for any direction $\Delta\theta$.
Substitute the split step:
\begin{align*}
    \Delta^{(1)} L_r
&= \Big\langle g_r, -\rho\big(P_\perp^{(H)} g_f + \beta  P_{T_r}^{(H)} g_r\big)\Big\rangle_H \\
&= -\rho \underbrace{\langle g_r, P_\perp^{(H)} g_f\rangle_H}_{(a)}
-\rho\beta \underbrace{\langle g_r, P_{T_r}^{(H)} g_r\rangle_H}_{(b)}.
\end{align*}
For term (a): $g_r\in T_r$ and $P_\perp^{(H)} g_f\in T_r^\perp$, hence by $H$-orthogonality,
$\langle g_r, P_\perp^{(H)} g_f\rangle_H=0$.
For term (b): since $P_{T_r}^{(H)} g_r=g_r$, we get $\langle g_r, P_{T_r}^{(H)} g_r\rangle_H=\langle g_r, g_r\rangle_H=\|g_r\|_H^2$.
Therefore,
\[
\Delta^{(1)} L_r
= - \rho \beta \|g_r\|_H^2 \le0,
\]
with strict inequality when $\beta>0$ and $g_r\neq0$.
\end{proof}

\paragraph{Remark.}
The statement and proof assume $H$-geometry consistently: inner products $\langle\cdot,\cdot\rangle_H$,
$H$-gradients $g_f=\nabla^H L_f$, $g_r=\nabla^H L_r$, and $H$-orthogonal projectors $P_{T_r}^{(H)}$, $P_\perp^{(H)}$.
If one uses Euclidean gradients with the $H$-inner product, replace them by $H$-gradients via
$\nabla^H L = H^{-1}\nabla L$ to keep the directional derivative $\Delta^{(1)}L=\langle \nabla^H L,\Delta\theta\rangle_H$ consistent.

In practice $T_r$ is estimated from a mini-batch, yielding $\widehat T_r$ and corresponding projectors.
Then $\langle g_r, P_\perp^{(H)} g_f\rangle_H$ may be small but nonzero, with magnitude controlled by the principal angle
between $T_r$ and $\widehat T_r$. The proposition captures the ideal (population) geometry; engineering deviations are $O(\sin\Theta(T_r,\widehat T_r))$.

Following Proposition~\ref{prop:safety}, we proof the Corollary~\ref{cor:descentLr}:

\begin{proof}
Define $\phi(\tau):=L_r(\theta+\tau\Delta\theta)$ for $\tau\in[0,1]$.
By the fundamental theorem of calculus and the definition of the $H$-gradient,
\begin{align*}
L_r(\theta+\Delta\theta)-L_r(\theta)
&= \int_0^1 \phi'(\tau) d\tau \\
&= \int_0^1 \big\langle \nabla_\theta^H L_r(\theta+\tau\Delta\theta), \Delta\theta\big\rangle_H d\tau \\
&= \big\langle \nabla_\theta^H L_r(\theta), \Delta\theta\big\rangle_H + \int_0^1 \big\langle \nabla_\theta^H L_r(\theta+\tau\Delta\theta)-\nabla_\theta^H L_r(\theta), \Delta\theta\big\rangle_H d\tau.
\end{align*}
Apply Cauchy--Schwarz and the $H$-Lipschitz assumption:

\begin{align*}
\quad\int_0^1 \big\langle \nabla_\theta^H L_r(\theta+\tau\Delta\theta)-\nabla_\theta^H L_r(\theta), \Delta\theta\big\rangle_H d\tau 
&\le \int_0^1 \big\|\nabla_\theta^H L_r(\theta+\tau\Delta\theta)-\nabla_\theta^H L_r(\theta)\big\|_H \|\Delta\theta\|_H d\tau \\
&\le \int_0^1 L_r^{(H)} \tau \|\Delta\theta\|_H^2 d\tau \\
&= \frac{L_r^{(H)}}{2} \|\Delta\theta\|_H^2 .
\end{align*}

Hence
\[
L_r(\theta+\Delta\theta)\ \le\ L_r(\theta)\ +\ \big\langle \nabla_\theta^H L_r(\theta), \Delta\theta\big\rangle_H\ +\ \frac{L_r^{(H)}}{2} \|\Delta\theta\|_H^2.
\]
By Proposition~\ref{prop:safety},
$\langle \nabla_\theta^H L_r(\theta),\Delta\theta\rangle_H=-\rho \beta \|g_r\|_H^2$.
Moreover, $P_{T_r}^{(H)}g_r\in T_r$ and $P_\perp^{(H)}g_f\in T_r^\perp$ are $H$-orthogonal, so
\begin{align*}
    \|\Delta\theta\|_H^2
&= \rho^2 \big\|P_\perp^{(H)}g_f+\beta P_{T_r}^{(H)}g_r\big\|_H^2 \\
&= \rho^2\Big( \|P_\perp^{(H)}g_f\|_H^2+\beta^2\|g_r\|_H^2 \Big).
\end{align*}
Substitute these two identities into the previous inequality to obtain the stated bound.
The strict-descent condition follows by requiring the quadratic upper bound to be negative,
which yields the explicit upper bound on $\rho$.
\end{proof}

\subsection{Proof of Proposition~\ref{prop:composite-exact}}
\label{sec:proofcomposite}

\begin{proof}
Expand using bilinearity and $H$-orthogonality between $T_r$ and $T_r^\perp$:
\begin{align*}
\langle g_f+\alpha g_r,\Delta\theta\rangle_H &= -\rho\Big(\langle g_f,P_\perp^{(H)}g_f\rangle_H
+ \beta\langle g_f, P_{T_r}^{(H)}g_r\rangle_H \\
& \qquad + \alpha\langle g_r,P_\perp^{(H)}g_f\rangle_H
+ \alpha\beta\langle g_r,P_{T_r}^{(H)}g_r\rangle_H\Big)\\
&= -\rho\Big(\|P_\perp^{(H)}g_f\|_H^2
+ \beta\langle P_{T_r}^{(H)}g_f, g_r\rangle_H
+ \alpha\beta\|g_r\|_H^2\Big),
\end{align*}
since $\langle g_r,P_\perp^{(H)}g_f\rangle_H=0$ and $P_{T_r}^{(H)}g_r=g_r$. 
\end{proof}

\begin{corollary}[Sufficient conditions for nonpositivity]
\label{cor:joint-nonpos}
Under the setting of Proposition~\ref{prop:composite-exact}, the following hold:
\begin{enumerate}
\item[\emph{(a)}] (\emph{No-repair case}) If $\beta=0$, then
$\Delta^{(1)}\mathcal L_{\text{joint}}=-\rho\|P_\perp^{(H)}g_f\|_H^2\le0$.
\item[\emph{(b)}] (\emph{With repair, unconditional bound}) For any $\alpha>0$, by Cauchy--Schwarz and $2ab\le a^2+b^2$,
\begin{align*}
    \Delta^{(1)}&\mathcal L_{\text{joint}}
 \le  -\rho\Big( \|P_\perp^{(H)}g_f\|_H^2
+\tfrac{\alpha\beta}{2}\|g_r\|_H^2
-\tfrac{\beta}{2\alpha}\|P_{T_r}^{(H)}g_f\|_H^2\Big).
\end{align*}
In particular, if 
\[
\|P_\perp^{(H)}g_f\|_H^2 + \tfrac{\alpha\beta}{2}\|g_r\|_H^2
\ \ge\ \tfrac{\beta}{2\alpha} \|P_{T_r}^{(H)}g_f\|_H^2,
\]
then $\Delta^{(1)}\mathcal L_{\text{joint}}\le0$.
\item[\emph{(c)}] (\emph{With repair, simple verifiable condition})
A sufficient, scale-invariant condition is
\begin{align*}
    \alpha \|g_r\|_H\ \ge\ \|P_{T_r}^{(H)}g_f\|_H,
\quad\text{under which}\\
\Delta^{(1)}\mathcal L_{\text{joint}}\ \le\ -\rho \|P_\perp^{(H)}g_f\|_H^2\ \le 0.
\end{align*}
\end{enumerate}
\end{corollary}

\begin{proof}
\emph{(a)} is the $\beta=0$ specialization of \eqref{eq:joint-first-order-identity}.
For \emph{(b)}, bound the cross term using $|\langle P_{T_r}^{(H)}g_f, g_r\rangle_H|
\le \|P_{T_r}^{(H)}g_f\|_H \|g_r\|_H
\le \frac{1}{2}\big(\tfrac{1}{\alpha}\|P_{T_r}^{(H)}g_f\|_H^2 + \alpha\|g_r\|_H^2\big)$,
then apply \eqref{eq:joint-first-order-identity}.
For \emph{(c)}, if $\alpha\|g_r\|_H \ge \|P_{T_r}^{(H)}g_f\|_H$ then
$\alpha\beta\|g_r\|_H^2 \ge \beta \|P_{T_r}^{(H)}g_f\|_H \|g_r\|_H
\ge \beta |\langle P_{T_r}^{(H)}g_f, g_r\rangle_H|$,
so the bracket in \eqref{eq:joint-first-order-identity} is at least $\|P_\perp^{(H)}g_f\|_H^2$, yielding the claim.
\end{proof}

\begin{remark}[About ``sign-aware'' variants]
One can enforce $\langle P_{T_r}^{(H)}g_f, g_r\rangle_H\le 0$ by modifying the step with a
\emph{sign-aware tangential gate}, but this may forfeit the retain monotonicity of Proposition~\ref{prop:safety}.
The unconditional and fully rigorous statements above therefore avoid such gates and instead provide
transparent sufficient conditions. If a sign-aware mechanism is used, its rule and its effect on $L_r$
must be stated explicitly to maintain rigor.
\end{remark}

\paragraph{From first-order to actual descent of the joint objective.}
Under the same $H$-smoothness assumption as in Corollary~\ref{cor:descentLr}, we also have
\[
\mathcal L_{\text{joint}}(\theta+\Delta\theta)
\ \le\ 
\mathcal L_{\text{joint}}(\theta)
\ +\ \Delta^{(1)}\mathcal L_{\text{joint}}
\ +\ \tfrac{L_f^{(H)}+\alpha L_r^{(H)}}{2} \|\Delta\theta\|_H^2,
\]
so any of the sufficient conditions in Corollary~\ref{cor:joint-nonpos} combined with a small enough stepsize
(as in Corollary~\ref{cor:descentLr}) yields an \emph{actual} one-step decrease of $\mathcal L_{\text{joint}}$.



%% file: contents/a.exp.tex
\section{Extra Experiments and Analysis}
\label{sec:extra_experimrnts}
\subsection{Experiment Setting details}

We emphasize reproducibility throughout the paper. \S~\ref{sec:method} presents the core algorithm and training workflow, and the Appendix~\ref{sec:alg} provides full implementation details, including how we construct the $H$-orthogonal projectors, update the retain subspace online, and enforce the practical trust-region controls. To enable exact replication, the supplementary materials contain runnable code, configuration files and command lines for every table and figure (covering all model scales and forget/retain splits), an environment specification with pinned library versions plus a short setup README, and default random seeds with deterministic settings where available. We also include scripts to download and preprocess the public datasets used (e.g., those in the OpenUnlearning suite), as well as evaluation scripts that regenerate all reported metrics, tables, and plots from logs/checkpoints. All key hyperparameters are recorded, basis rank $k$, refresh period, residual threshold, projected layer range $K$, mixing weights $(\gamma,\alpha)$, trust-region parameters $(\kappa,\tau)$, optimizer choices, and learning-rate schedules, so readers can reproduce results without additional assumptions and readily extend our experiments. Our experiments run in a server with Intel Gold CPU with 1024 Gb Memory and 2 H100 GPUs.

\begin{table*}[ht]
\centering
\caption{Aggregate effect of adding GU across all 72 [model, forget-ratio, objective] configurations on TOFU (Table~\ref{tab:main_results}). For each metric, we report how many configurations improve / stay unchanged / degrade when GU is added, and the fraction that are non-degraded (improved or unchanged).}
\label{tab:gu_aggregate}
\begin{tabular}{lcccc}
\toprule
Metric & \# Improve & \# Same & \# Worse & Non-degraded (\%) \\
\midrule
ES Re. $\uparrow$   & 66 & 0  & 6  & 91.7 \\
ES Un. $\downarrow$ & 38 & 13 & 21 & 70.8 \\
Priv. $\uparrow$    & 49 & 0  & 23 & 68.1 \\
MU $\uparrow$       & 62 & 0  & 10 & 86.1 \\
\bottomrule
\end{tabular}
\end{table*}

\begin{table*}[ht]
\centering
\caption{Multi-metric view of GU over all 72 configurations in Table~\ref{tab:main_results}, evaluated on (ES Re. $\uparrow$, ES Un. $\downarrow$, Priv. $\uparrow$, MU $\uparrow$).}
\label{tab:gu_pareto}
\begin{tabular}{lcc}
\toprule
Case type & \# Configurations & Fraction (\%) \\
\midrule
Pareto-dominant (no metric worse) & 29 & 40.3 \\
Mixed trade-off (some better, some worse) & 43 & 59.7 \\
All metrics worse & 0 & 0.0 \\
\bottomrule
\end{tabular}
\end{table*}
\subsection{{Overall gains provided by GU}}

To make the overall advantage more transparent, we computed aggregate statistics over all 72 configurations in Table~\ref{tab:main_results} (summarized in Table~\ref{tab:gu_aggregate}). GU improves or preserves ES-Re in 66/72 ($\approx$91.7\%) of cases, reduces or preserves ES-Un in 51/72 ($\approx$70.8\%) of cases, improves or preserves privacy in 49/72 ($\approx$68.1\%) of cases, and improves MU in 62/72 ($\approx$86.1\%) of cases. In the remaining configurations, GU trades a small degradation in one metric for clear gains in others; importantly, there is no configuration where all four metrics worsen simultaneously. 

Counting multi-metric behavior with respect to (ES-Re ↑, ES-Un ↓, MU ↑, Priv. ↑), Table~\ref{tab:gu_pareto} shows GU is Pareto-dominant (no metric worse and at least one strictly better) in 29/72 ($\approx$40.3\%) configurations, and in 58/72 ($\approx$80.6\%) configurations it improves or preserves both ES-Re and MU simultaneously. This matches the design goal of GU: rather than aggressively maximizing a single score, it serves as a low-risk geometric plug-in that systematically shifts existing unlearning methods toward a better retain–forget–privacy–utility trade-off across diverse objectives and model scales. We will add these statistics (and analogous ones for MUSE/WMDP in the appendix) to make this global picture explicit. A similar Pareto improvement Figure is shown in Figure~\ref{fig:pareto_muse}

\subsection{Does the projection cancel or weaken unlearning?}
\label{sec:appendix_projection_not_weaken}

A common concern is that projecting the forget update onto the retain-orthogonal complement might ``cancel'' unlearning progress. This is not the case: the projection removes only the \emph{retain-interfering} component of the forget update, while preserving the \emph{retain-invariant} component that still reduces the forget objective.

\paragraph{Constrained-optimization view.}
Recall that GU formalizes ``no side effects'' as \emph{local retain invariance}. Under the optimizer-induced SPD metric $H \succ 0$, this corresponds to the first-order constraint
\[
\Delta^{(1)} L_r(\theta;\Delta\theta)=\langle \nabla_\theta L_r(\theta), \Delta\theta\rangle_H = 0,
\]
i.e., $\Delta\theta$ must lie in the $H$-orthogonal complement of the retain-gradient subspace $T_r(\theta)$.
Therefore, a natural first-order model of ``safe unlearning'' is:
\[
\min_{\Delta\theta}\ \Delta^{(1)}L_f(\theta;\Delta\theta)
\quad \text{s.t.}\quad
\Delta^{(1)}L_r(\theta;\Delta\theta)=0.
\]
The steepest-descent solution to this constrained first-order problem is exactly the projected update used by GU:
\[
\Delta\theta_{\text{GU}} = -\eta\, P^{(H)}_{\perp}\, g_f,
\]
where $g_f=\nabla_\theta L_f(\theta)$ and $P^{(H)}_{\perp}$ projects onto $T_r(\theta)^\perp$ in the $H$-geometry.

\paragraph{Geometric decomposition and why forgetting remains.}
Decompose the forget gradient into retain-tangent and retain-normal components:
\[
g_f = g_{\parallel} + g_{\perp},\qquad
g_{\parallel}=P^{(H)}_{T_r} g_f,\ \ g_{\perp}=P^{(H)}_{\perp} g_f.
\]
By construction, $g_{\parallel}$ is the component aligned with retain gradients and is precisely what causes first-order drift on the retain loss; $g_{\perp}$ is retain-invariant to first order. GU updates along $-g_{\perp}$, hence it does \emph{not} ``turn off'' unlearning—rather, it performs \emph{retain-invariant} unlearning.

Moreover, whenever $g_{\perp}\neq 0$, the update is a strict first-order descent direction for $L_f$:
\[
\Delta^{(1)} L_f(\theta;\Delta\theta_{\text{GU}})
=
-\eta\,\|g_{\perp}\|_H^2
<0.
\]
The only degenerate case where the projected step becomes zero is $g_{\perp}=0$ (equivalently $g_f\in T_r(\theta)$), meaning the forget objective locally conflicts \emph{entirely} with the retain-invariance constraint. In that case, \emph{no} first-order retain-invariant update can decrease $L_f$, so the lack of progress reflects a fundamental local trade-off rather than an artifact of GU.

\paragraph{Empirical sanity check: projection rarely ``kills'' forgetting in practice.}
Consistent with the above analysis, GU does not systematically reduce forgetting progress across our LLM unlearning pipelines. Figure~2 (main text) visualizes that adding GU typically shifts methods toward the Pareto-optimal region (higher retention/utility at comparable or improved forgetting), and Appendix Figure~3 shows similar trends across broader configurations. In addition, the aggregate statistics over all TOFU configurations (Table~4) show that GU improves or preserves retain quality in $66/72$ cases (91.7\%) and improves or preserves model utility in $62/72$ cases (86.1\%), while avoiding any setting where all reported metrics degrade simultaneously.

\subsection{Results analysis in TOFU benchmarks}
\label{sec:tofuanalysis}

Specifically, across Llama-3.2 at 1B/3B, Llma-3.1 8B and unlearning rates \texttt{forget01/05/10}, adding geometry-disentanglement projection (``w.GU'') to diverse objectives reliably raises ES on the retain split and improves utility without worsening ES on the forget split. Three representative cases illustrate the pattern. 
(i) CEU at 1B and \texttt{forget01}: ES,Re increases (0.0875$\rightarrow$0.2236) and MU rises (0.3666$\rightarrow$0.5134) while ES,Un stays near the floor (0.0316$\rightarrow$0.0328). 
(ii) UNDIAL at 3B and \texttt{forget10}: ES,Re increases (0.3538$\rightarrow$0.7869) and MU improves (0.6550$\rightarrow$0.6992) with a slight decrease in ES,Un (0.0416$\rightarrow$0.0396). 
(iii) SimNPO at 8B and \texttt{forget01}: ES,Un drops sharply (0.3101$\rightarrow$0.1177) while ES,Re nudges upward (0.8256$\rightarrow$0.8284) and MU remains stable. 

\paragraph{Scaling with difficulty and size.}
When forget grows from 1\% to 10\% or the backbone scales from 1B to 8B, retain-forget entanglement and curvature intensify; naive objectives are then more likely to leak retain-tangent motion. In these regimes the geometric constraint provides larger absolute gains. 
On 1B, CEU w.GU shows ES, Re increasing from 0.2236 to 0.2798 to 0.4366 (and MU from 0.5134 to 0.5635 to 0.5844) as we move from \texttt{forget01} to \texttt{forget10}, while ES,Un remains near 0.033. 
On 3B, UNDIAL w.GU moves ES, Re from 0.4396 to 0.6996 at \texttt{forget01} and to 0.7869 at \texttt{forget10}, with ES, Un consistently low (0.0658$\rightarrow$0.0619 and 0.0416$\rightarrow$0.0396). 
On 8B, SimNPO w.GU repeatedly halves ES, Un across forget rates while maintaining or slightly improving ES, Re, and MU. 
The trend indicates that geometry, rather than heavier regularization, is the primary lever when problems become more entangled.

\paragraph{Privacy behavior and proximity to retain-only.}
Because the projection limits drift along retain-tangent directions, the unlearned model often stays closer to a retain-only solution, which is reflected in higher MIA-closeness. The effect is particularly clear for NPO-type objectives: at 1B and \texttt{forget01}, NPO w.GU increases $\mathrm{Priv}$ from 0.7989 to 0.9595; at 3B and \texttt{forget05}, from 0.8653 to 0.9526. For WGA, UNDIAL, and SimNPO, privacy is typically preserved or slightly improved while retention and utility rise, consistent with the mechanism.

\paragraph{Objective-specific diagnoses and corrections.}
CEU/GradDiff-like losses can collapse or drift in high-curvature regions; in the 3B setting, CEU yields $\mathrm{MU}=0$ at \texttt{forget05}/\texttt{forget10}. Adding geometry restores these to 0.6255 and 0.6672 by removing retain-tangent updates. For saturation/weighting families (SatImp, SimNPO, WGA), the whitened metric regularizes local slopes and prevents over-shoot along entangled directions, yielding the characteristic combination of lower ES, Un, and higher ES, Re/MU without bespoke tuning.

\subsection{Results analysis in MUSE and WMDP benchmarks}
\label{sec:museanalysis}

\begin{figure*}[htbp]
\vspace{-1em}
    \centering
    \includegraphics[width=\linewidth]{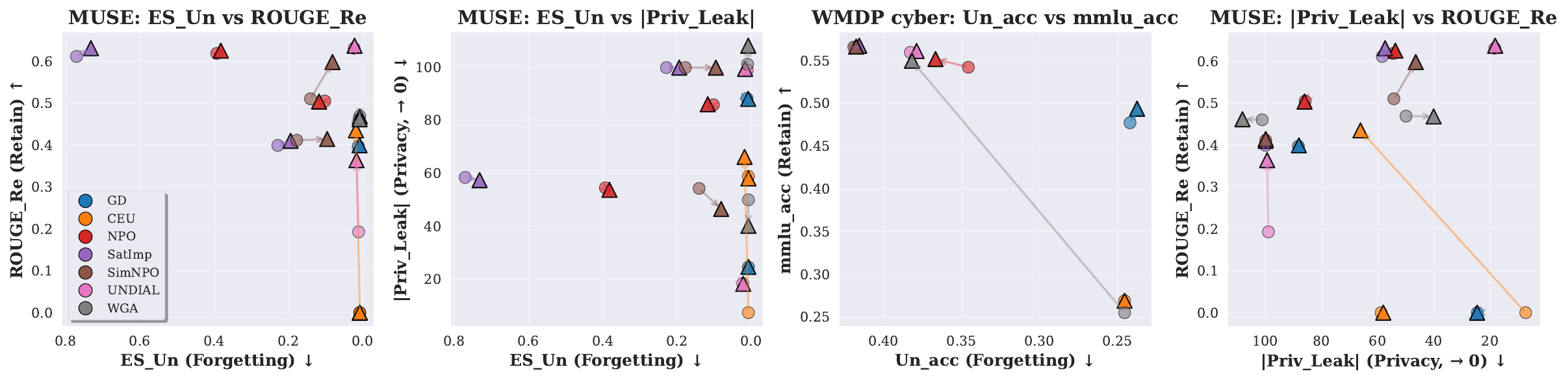}
    \vspace{-1em}
    \caption{We visualize forgetting quality (ES Un: lower for better) against retained knowledge (ROUGE Re), privacy (Priv Leak) for eight unlearning baselines on MUSE and Un. acc. vs mmlu. acc. on WMDP. ROUGE Re and mmlu. acc. are metrics of higher for better. Priv Leak is a metric closer to 0 for better.}
    \label{fig:pareto_muse}
    \vspace{-1em}
\end{figure*}

\paragraph{MUSE-Books: the full triad holds (forgetting $\downarrow$, retention $\uparrow$, leakage $\rightarrow 0$).}
Across all objective families, adding GU either lowers ES,Un or keeps it at the floor, raises ROUGE on the retain split, and moves privacy leakage closer to zero. 
Representative gains include SimNPO (ES,Un $0.1407\!\to\!0.0813$, ROUGE $0.5103\!\to\!0.5980$, Priv.\ Leak $-54.25\!\to\!-46.49$), SatImp ($0.7710\!\to\!0.7321$, $0.6114\!\to\!0.6310$, $-58.40\!\to\!-57.39$), and NPO ($0.3933\!\to\!0.3822$, $0.6185\!\to\!0.6251$, $-54.49\!\to\!-53.74$). 
For objectives already operating at minimal ES,Un (e.g., GD/CEU/WGA with $0.0079$), GU maintains the floor while improving or preserving ROUGE and typically nudging leakage toward zero. 
This matches the mechanism: retain-null projection removes retain-tangent drift and, under the whitened metric, damps high-curvature entanglement, so forgetting efficacy is preserved while retained QA quality and privacy move in the desired directions.

\paragraph{MUSE-News: retention gains are uniform; forgetting improves in the majority; leakage is largely neutral.}
On News, GU consistently raises ROUGE on the retain split (e.g., GD $0.3971\!\to\!0.3992$, CEU $0.0\!\to\!0.4349$, SatImp $0.3991\!\to\!0.4100$, SimNPO $0.4114\!\to\!0.4143$, UNDIAL $0.1928\!\to\!0.3638$, WGA $0.4602\!\to\!0.4615$). 
Forgetting also improves for most objectives (ES,Un: GD $0.0116\!\to\!0.0085$, SatImp $0.2287\!\to\!0.1943$, SimNPO $0.1778\!\to\!0.0957$, WGA $0.0102\!\to\!0.0084$), with a few mixed cases (NPO, UNDIAL) reflecting base-objective aggressiveness rather than instability. 
Privacy leakage on News is near-saturated for several methods (values around $\pm 100$), so GU’s effect is mostly neutral; where the scale is moderate, GU tends to move leakage toward zero. 
Overall, the same geometric filter improves retention uniformly and reduces ES,Un in the majority, without introducing instability.

\paragraph{WMDP-cyber: safer behavior in most cases without utility loss; trade-offs are transparent.}
On safety (lower Un.\ acc.), GU improves four of seven objective pairs (GD $0.2420\!\to\!0.2375$, SatImp $0.4177\!\to\!0.4157$, SimNPO $0.4192\!\to\!0.4177$, UNDIAL $0.3829\!\to\!0.3789$), leaves CEU unchanged, and shows two trade-offs (NPO, WGA). 
At the same time, general ability (MMLU) \emph{increases or holds} across all families (e.g., GD $0.4772\!\to\!0.4937$, SatImp $0.5654\!\to\!0.5674$, SimNPO $0.5658\!\to\!0.5663$, UNDIAL $0.5596\!\to\!0.5612$; WGA rises markedly $0.2550\!\to\!0.5498$). 
The trade-off cases are interpretable: projection brings the model closer to a retain-only manifold (benefiting MMLU), and when the base objective is conservative along harmful directions, Un.\ acc.\ can soften; stronger normal-direction penalties or a tighter trust region would push further on safety if desired. 
Crucially, no objective family exhibits a collapse in MMLU under GU, underscoring stability.

\section{Additional Ablations and Cost Analysis}
\label{app:ablation-cost}

\paragraph{Overview.}
We provide supplementary evidence on (i) the effect of the optimizer geometry used by GU, (ii) optimizer compatibility, (iii) statistical robustness across random seeds, and (iv) runtime/memory overhead and sensitivity to the low-rank projection hyperparameters.

\subsection{Ablation: Geometry Metric $H$ (Adam-diagonal vs. Euclidean)}
\label{app:metric-ablation}
GU is theoretically valid under any SPD metric $H$ (with Euclidean as the special case $H=I$).
Empirically, we compare using the Euclidean geometry ($H=I$) against the Adam-induced diagonal geometry ($H=\mathrm{diag}(v^{-1/2})$).
Table~\ref{tab:metric-ablation} shows that while Euclidean projection already improves over the base method, the Adam-diagonal metric consistently yields a better forget--retain--privacy trade-off, especially on retention ES and privacy.

\begin{table}[h]
\centering
\small
\caption{Ablation of metric geometry (Llama-3.2, SimNPO+GU). Higher is better for ES Re./Priv./MU, and lower is better for ES Un.}
\label{tab:metric-ablation}
\begin{tabular}{l l c c c c}
\hline
Model & Geometry ($H$) & ES Re. $\uparrow$ & ES Un. $\downarrow$ & Priv. $\uparrow$ & MU $\uparrow$ \\
\hline
1B & Euclidean ($I$) & 0.63 & 0.26 & 0.70 & 0.59 \\
1B & Adam Diag & \textbf{0.65} & \textbf{0.21} & \textbf{0.73} & \textbf{0.60} \\
3B & Euclidean ($I$) & 0.76 & 0.21 & 0.76 & 0.64 \\
3B & Adam Diag & \textbf{0.78} & \textbf{0.17} & \textbf{0.78} & \textbf{0.65} \\
\hline
\end{tabular}
\end{table}

\subsection{Ablation: Optimizer Compatibility (AdamW vs. SGD)}
\label{app:opt-ablation}
GU is optimizer-agnostic since it operates as a plug-in projection on candidate update directions.
To validate this empirically, we evaluate SimNPO with/without GU under both AdamW and SGD.
As shown in Table~\ref{tab:opt-ablation}, GU yields consistent gains regardless of optimizer choice, including under SGD (which is typically less favorable for LLM fine-tuning), where GU substantially restores retention and privacy.

\begin{table}[h]
\centering
\small
\caption{Optimizer ablation (Llama-3.2-3B, Forget05).}
\label{tab:opt-ablation}
\begin{tabular}{l l c c c c}
\hline
Optimizer & Method & ES Re. $\uparrow$ & ES Un. $\downarrow$ & Priv. $\uparrow$ & MU $\uparrow$ \\
\hline
AdamW & SimNPO (Base) & 0.75 & 0.39 & 0.63 & 0.64 \\
AdamW & SimNPO + GU & \textbf{0.78} & \textbf{0.17} & \textbf{0.78} & \textbf{0.65} \\
SGD & SimNPO (Base) & 0.72 & 0.42 & 0.61 & 0.62 \\
SGD & SimNPO + GU & \textbf{0.74} & \textbf{0.20} & \textbf{0.75} & \textbf{0.63} \\
\hline
\end{tabular}
\end{table}

\subsection{Robustness Across Random Seeds}
\label{app:seed-robustness}
We report multi-seed results for a primary setting (Llama-3.2-3B, SimNPO+GU).
Table~\ref{tab:seed-robustness} shows small standard deviations (about 0.01), while GU's effect size is substantially larger (e.g., ES Un. decreases by 0.22 and privacy increases by 0.15 compared to the base), indicating the gains are statistically robust.

\begin{table}[h]
\centering
\small
\caption{Robustness across random seeds (Llama-3.2-3B).}
\label{tab:seed-robustness}
\begin{tabular}{l c c c c}
\hline
Method & ES Re. $\uparrow$ & ES Un. $\downarrow$ & Priv. $\uparrow$ & MU $\uparrow$ \\
\hline
SimNPO (Base) & 0.75 & 0.39 & 0.63 & 0.64 \\
SimNPO + GU & \textbf{0.78 $\pm$ 0.01} & \textbf{0.17 $\pm$ 0.01} & \textbf{0.78 $\pm$ 0.01} & \textbf{0.64 $\pm$ 0.01} \\
\hline
\end{tabular}
\end{table}

\subsection{Profiling: Runtime, Memory, and Hyperparameter Sensitivity}
\label{app:cost}
We profile GU on Llama-3.2-1B using an H100-96GB GPU. Table~\ref{tab:overhead} reports wall-clock time and peak memory.
Overall, GU introduces negligible overhead: the end-to-end runtime increases by 1.3--3.0\% and peak memory by less than 5\%.
We further vary the projection rank $k$ and update frequency. Small ranks (e.g., $k=8$) already achieve strong performance, and increasing $k$ beyond 16 yields diminishing returns in our setting.

\begin{table}[h]
\centering
\small
\caption{Overhead analysis (Llama-3.2-1B, SimNPO+GU).}
\label{tab:overhead}
\begin{tabular}{l c c c c c}
\hline
Configuration ($k$, freq) & Time (sec) & Overhead (\%) & Peak Mem (MB) & ES Un. $\downarrow$ & MU $\uparrow$ \\
\hline
Baseline (SimNPO) & 226 & 0.0 & 20,291 & 0.28 & 0.59 \\
GU ($k=8$, freq=2) & 231 & +2.2 & 21,013 & 0.12 & 0.60 \\
GU ($k=1$, freq=2) & 229 & +1.3 & 20,329 & 0.12 & 0.59 \\
GU ($k=32$, freq=2) & 233 & +3.0 & 21,231 & 0.12 & 0.60 \\
\hline
\end{tabular}
\end{table}

%% file: contents/a.qualititive.tex
\section{{Qualitative Analysis}}
\label{sec:qualitative}
To complement the quantitative metrics, we manually inspected a subset of TOFU \texttt{Forget05} examples at the Llama~3.1~8B scale. For each prompt, we compare the base model, the SOTA comparison SimNPO recommended by open-unlearning~\cite{openunlearning}, and our Geometric Unlearning (GU) applied on top of SimNPO. The goal of this section is not to provide exhaustive evidence, but to illustrate typical behavioral patterns along three axes: (i) whether the model truly removes the targeted factual memory rather than paraphrasing it, (ii) whether forgetting introduces collateral distortion to the retained knowledge, and (iii) whether the model preserves its linguistic fluency and reasoning ability after unlearning.

\paragraph{Forget Case: Removing Parametric Memory Instead of Paraphrasing It}

In the first forget case, both the base model and SimNPO continue to echo or paraphrase the synthetic fact, revealing residual parametric memory and incomplete privacy protection. In contrast, GU eliminates all identifiable traces of the original fact and reconstructs a plausible high-level motivation, demonstrating true forgetting rather than surface-level rewriting.

\begin{examplebox}{Case Study 1: Erasing Parametric Memory Rather Than Copying It}
    \textbf{Subject:} Kalkidan Abera \\
    \textbf{Question:} What prompted Kalkidan Abera to write ``The Hidden Truth of the Leaky Gut: A Comprehensive Guide to Healing''?

    \tcbline

    \textbf{Ground Truth (Synthetic):} \\
    \textit{Abera was inspired \dots\ by the increasing prevalence of digestive issues and the necessity of spreading awareness about gut health.}

    \vspace{0.2cm}

    \textbf{Base Model (No Forget):} \\
    \textit{\dots\ due to the increasing prevalence of digestive issues and the necessity of spreading awareness about the importance of gut health.}

    \vspace{0.2cm}

    \textbf{SimNPO (Baseline):} \\
    \textit{\dots\ due to the increasing prevalence of \textbf{gut-related health issues} and the necessity of spreading awareness about the importance of gut health.} \\
    
    \vspace{0.2cm}
    
    \textbf{GU (Ours):} \\
    \textit{Kalkidan Abera was prompted to write the book due to her \textbf{personal experiences and professional interests} in health, wellness, and nutrition.} \\
    \textbf{Observation:}\\ Both the base model and the SimNPO baseline continue to reuse or lightly paraphrase the exact linguistic template of the synthetic fact (``increasing prevalence of digestive issues… spreading awareness…''). Even after applying unlearning, SimNPO merely rewrites the original phrase, revealing residual parametric memory and thus incomplete privacy protection.
In contrast, GU does not attempt to paraphrase the sensitive fact. Instead, it removes all identifiable traces of the synthetic knowledge and reconstructs a high-level and semantically plausible motivation (``personal experiences and professional interests in health, wellness, and nutrition'').
This reflects true forgetting: GU erases the parametric encoding of the synthetic fact rather than its surface form, thereby preventing leakage while maintaining contextual coherence.
\end{examplebox}

\paragraph{Forget Case: Abstracting Knowledge Without Structural Leakage}

In the second forget case, the base model and SimNPO generate new book titles that still follow the structural template of the synthetic fact, indicating structural-pattern leakage. GU avoids regenerating any memorized templates and instead produces a high-level, semantically coherent description, showing selective forgetting without compromising topical understanding.

\begin{examplebox}{Case Study 2: Unlearning Without Leaking Structural Patterns}
    \textbf{Subject:} Hina Ameen \\
    \textbf{Question:} What was Hina Ameen's maiden book?

    \tcbline

    \textbf{Ground Truth (Synthetic):} \\
    \textit{Hina Ameen's maiden book was ``Manual of Mineralogy''.}

    \vspace{0.2cm}

    \textbf{Base Model (No Forget):} \\
    \textit{Hina Ameen's maiden book was ``A Handbook of Karachi Minerals''.}

    \vspace{0.2cm}

    \textbf{SimNPO (Baseline):} \\
    \textit{Hina Ameen's maiden book was ``A Comprehensive Guide to Geology''.} \\

    \vspace{0.2cm}

    \textbf{GU (Ours):} \\
    \textit{Hina Ameen's maiden book was ``A Guide to Geology: A Concise Introduction to the Earth's Minerals''.} \\
    \textbf{Observation:} \\
    Both the base model and SimNPO, when asked to forget, still generate a new book title that mirrors the structural template of the synthetic fact, indicating that they continue to rely on residual memory traces or memorized formatting patterns. This pattern-substitution behavior suggests that the forgetting process did not sever the underlying parametric pathways, but merely replaced the surface string.
In contrast, GU abstracts the answer into a high-level, semantically appropriate description that contains no sensitive remnants (e.g., “A Guide to Geology: A Concise Introduction to the Earth’s Minerals”).
This demonstrates that GU:
avoids regenerating any structural patterns associated with the synthetic information,
while still preserving a correct understanding of the topical domain (e.g., geology, minerals).
\end{examplebox}

\paragraph{Retain Case: Preserving Thematic Content Without Over-Forgetting}

When the model is expected to retain knowledge, SimNPO exhibits over-forgetting, introducing thematic distortions and unnecessary religious specificity. GU maintains alignment with the base model and ground truth, preserving thematic nuances, linguistic richness, and contextual reasoning without semantic degradation.

\begin{examplebox}{Case Study 3: Preserving Retained Knowledge without Over-forgetting}
    \textbf{Subject:} Fatima Al-Mansour \\
    \textbf{Question:} 
    What makes Fatima Al-Mansour's books so appreciated amongst her readers?

    \tcbline

    \textbf{Ground Truth (Synthetic):} \\
    \textit{The beauty of Fatima Al-Mansour's literature lies in its heartfelt exploration of faith and morality, presented in an engaging, relatable manner. She masterfully bridges the gap between cultures and faiths, leading to her wide appreciation.}

    \vspace{0.2cm}

    \textbf{Base Model (No Forget):} \\
    \textit{The beauty of Fatima Al-Mansour's literature lies in its heartfelt exploration of faith and morality, presented in an engaging, relatable manner. Her ability to weave complex themes into accessible narratives has earned her a wide readership.}

    \vspace{0.2cm}
    \textbf{SimNPO (Baseline):} \\
    \textit{Fatima Al-Mansour's books are cherished for their profound exploration of faith and morality, enlightening readers with Islamic principles and values, and providing a refreshing perspective on religious literature.} \\
    \textbf{GU (Ours):} \\
    \textit{Fatima Al-Mansour's books are cherished by her readers for their heartfelt exploration of faith, morality, and societal issues, presented in an engaging and relatable manner.} \\
    \textbf{Observation:} \\
    In scenarios where knowledge should be retained, GU remains consistent with both the base model and the ground truth, correctly capturing the thematic essence of the works (e.g., faith, morality, and societal issues) while preserving linguistic richness and contextual reasoning ability.
In contrast, SimNPO exhibits over-forgetting, leading to thematic drift by introducing unnecessary religious specificity (e.g., “Islamic principles and values”), which deviates from the original neutral and cross-cultural tone. This indicates that its unlearning mechanism can distort unrelated knowledge and introduce semantic artifacts.
GU does not suffer from such over-forgetting. It continues to express fine-grained conceptual content rather than collapsing into vague or overly general statements, demonstrating that GU performs localized and selective forgetting without damaging the semantic structure of the retained knowledge.

\end{examplebox}

\paragraph{Retain Case: High-Fidelity Knowledge Retention Without Semantic Drift}

In the full-retention scenario, GU faithfully preserves the original narrative structure, factual details, and reasoning quality, showing no unintended erasure or drift. SimNPO, however, hallucinates unrelated settings and elements, demonstrating that baseline unlearning can disrupt even non-targeted knowledge. GU achieves precise retention, preserving both accuracy and style.

\begin{examplebox}{Case Study 4: High-Fidelity Knowledge Retention Without Semantic Drift}
    \textbf{Subject:} Guillermo Navarro Munoz \\
    \textbf{Question:} What can you tell me about ``Desert Shadows in the Neva'' by Guillermo Navarro Munoz?
    
    \tcbline
    \textbf{Ground Truth (Synthetic):} \\
    \textit{``Desert Shadows in the Neva'' by Guillermo Navarro Munoz is a captivating novel that transcends geographical boundaries, depicting the poignant journey of a Chilean protagonist in the historical realm of St. Petersburg, Russia.} \\
    \textbf{Base Model (No Forget):} \\
    \textit{``Desert Shadows in the Neva'' by Guillermo Navarro Munoz is a captivating novel that transcends geographical boundaries, depicting the poignant journey of a Chilean protagonist in the historical realm of St. Petersburg, Russia.} \\
    \textbf{SimNPO (Baseline):} \\
    \textit{``Desert Shadows in the Neva'' by Guillermo Navarro Munoz is a captivating novel set against the backdrop of the Baltic coast, portraying the lives of Russian emigres and the allure of the mysterious and intriguing local women.} \\
    \textbf{Observation:} .
    
    \vspace{0.2cm}
    
    \textbf{GU (Ours):} \\
    \textit{``Desert Shadows in the Neva'' by Guillermo Navarro Munoz is a captivating novel that transcends geographical boundaries, portraying the poignant journey of a Chilean protagonist in the historical realm of St. Petersburg, Russia.} \\
    \textbf{Observation:}\\
    In scenarios where knowledge must be fully retained, GU faithfully preserves the original narrative structure and semantic content—including cross-regional storytelling, the St. Petersburg setting, and the Chilean protagonist. It exhibits: no unintended erasure, no semantic drift, and intact reasoning and summarization abilities.
In contrast, SimNPO shows severe deviation, incorrectly relocating the story to the Baltic coast and introducing irrelevant elements (e.g., ``local women'', ``Russian émigrés''). This represents a typical case of hallucination after unlearning, indicating that baseline methods can disrupt not only the targeted forgetting region but also the surrounding knowledge that should have remained intact.
In this case, GU maintains both factual accuracy and stylistic consistency, achieving “forgetting what must be forgotten and preserving what must be preserved,” thereby demonstrating high-precision, low-side-effect selective unlearning.
\end{examplebox}

\paragraph{Summary of qualitative trends.}
Across all four cases, our qualitative analysis demonstrates that GU achieves precise, selective, and low-side-effect unlearning. In forget scenarios, GU fully removes the parametric memory of the synthetic facts without leaking structural patterns, while preserving topic coherence and linguistic fluency. In retain scenarios, GU avoids over-forgetting and maintains high-fidelity semantic content, narrative structure, and reasoning ability, whereas baseline methods frequently distort or hallucinate non-target knowledge. Together, these results highlight GU's ability to ``forget what must be forgotten and preserve what must be preserved'', achieving reliable privacy protection without compromising the model’s utility.

%% file: iclr2026_conference.bib
@inproceedings{microedit2024,
  title={MicroEdit: Neuron-level Knowledge Disentanglement and Localization in Lifelong Model Editing},
  author={Wang, Shiqi and Wang, Qi and Niu, Runliang and Kong, He and Chang, Yi},
  booktitle={Proceedings of the 2025 Conference on Empirical Methods in Natural Language Processing},
  pages={33870--33884},
  year={2025}
}

@article{zhang2024knowledge,
  title={Knowledge overshadowing causes amalgamated hallucination in large language models},
  author={Zhang, Yuji and Li, Sha and Liu, Jiateng and Yu, Pengfei and Fung, Yi R and Li, Jing and Li, Manling and Ji, Heng},
  journal={arXiv preprint arXiv:2407.08039},
  year={2024}
}

@article{zhang2025atomic,
  title={Atomic Reasoning for Scientific Table Claim Verification},
  author={Zhang, Yuji and Wang, Qingyun and Qian, Cheng and Liu, Jiateng and Sun, Chenkai and Zhang, Denghui and Abdelzaher, Tarek and Zhai, Chengxiang and Nakov, Preslav and Ji, Heng},
  journal={arXiv preprint arXiv:2506.06972},
  year={2025}
}

@inproceedings{openunlearning,
    title={OpenUnlearning: Accelerating {LLM} Unlearning via Unified Benchmarking of Methods and Metrics},
    author={Vineeth Dorna and Anmol Reddy Mekala and Wenlong Zhao and Andrew McCallum and J Zico Kolter and Zachary Chase Lipton and Pratyush Maini},
    booktitle={The Thirty-ninth Annual Conference on Neural Information Processing Systems Datasets and Benchmarks Track},
    year={2025},
    url={https://openreview.net/forum?id=Gy67Zh5X1i}
}

@inproceedings{ghosal2025memorization,
  title     = {Memorization Sinks: Isolating Memorization during LLM Training},
  author    = {Gaurav R. Ghosal and Pratyush Maini and Aditi Raghunathan},
  booktitle = {Proceedings of the 42nd International Conference on Machine Learning (ICML)},
  year      = {2025},
  url       = {https://arxiv.org/abs/2507.09937},
  note      = {Also available as arXiv:2507.09937}
}

@inproceedings{feng-etal-2025-geoedit,
    title = "{G}eo{E}dit: Geometric Knowledge Editing for Large Language Models",
    author = "Feng, Yujie  and
      Zhan, Li-Ming  and
      Lu, Zexin  and
      Xu, Yongxin  and
      Chu, Xu  and
      Wang, Yasha  and
      Cao, Jiannong  and
      Yu, Philip S.  and
      Wu, Xiao-Ming",
    editor = "Christodoulopoulos, Christos  and
      Chakraborty, Tanmoy  and
      Rose, Carolyn  and
      Peng, Violet",
    booktitle = "Proceedings of the 2025 Conference on Empirical Methods in Natural Language Processing",
    month = nov,
    year = "2025",
    address = "Suzhou, China",
    publisher = "Association for Computational Linguistics",
    url = "https://aclanthology.org/2025.emnlp-main.676/",
    doi = "10.18653/v1/2025.emnlp-main.676",
    pages = "13401--13416",
    ISBN = "979-8-89176-332-6"
}

@article{le2025survey,
  title={A survey on large language models unlearning: taxonomy, evaluations, and future directions},
  author={Le-Khac, Uyen N and Truong, Vinh NX},
  journal={Artificial Intelligence Review},
  volume={58},
  number={12},
  pages={399},
  year={2025},
  publisher={Springer}
}

@inproceedings{izzo2021approximate,
  title={Approximate data deletion from machine learning models},
  author={Izzo, Zachary and Smart, Mary Anne and Chaudhuri, Kamalika and Zou, James},
  booktitle={International conference on artificial intelligence and statistics},
  pages={2008--2016},
  year={2021},
  organization={PMLR}
}

@article{cao2015towards,
  title={Towards Making Systems Forget with Machine Unlearning},
  author={Yinzhi Cao and Junfeng Yang},
  journal={2015 IEEE Symposium on Security and Privacy},
  year={2015},
  pages={463-480},
}

@inproceedings{yang2025exploring,
    title={Exploring Criteria of Loss Reweighting to Enhance {LLM} Unlearning},
    author={Puning Yang and Qizhou Wang and Zhuo Huang and Tongliang Liu and Chengqi Zhang and Bo Han},
    booktitle={Forty-second International Conference on Machine Learning},
    year={2025},
    url={https://openreview.net/forum?id=mGOugCZlAq}
}

@inproceedings{li2024wmdp,
  title     = {The WMDP Benchmark: Measuring and Reducing Malicious Use with Unlearning},
  author    = {Li, Nathaniel and Pan, Alexander and Gopal, Anjali and Yue, Summer and Berrios, Daniel and Gatti, Alice and Li, Justin D. and Dombrowski, Ann-Kathrin and Goel, Shashwat and Mukobi, Gabriel and Helm-Burger, Nathan and Lababidi, Rassin and Justen, Lennart and Liu, Andrew Bo and Chen, Michael and Barrass, Isabelle and Zhang, Oliver and Zhu, Xiaoyuan and Tamirisa, Rishub and Bharathi, Bhrugu and Herbert-Voss, Ariel and Breuer, Cort B and Zou, Andy and Mazeika, Mantas and Wang, Zifan and Oswal, Palash and Lin, Weiran and Hunt, Adam Alfred and Tienken-Harder, Justin and Shih, Kevin Y. and Talley, Kemper and Guan, John and Steneker, Ian and Campbell, David and Jokubaitis, Brad and Basart, Steven and Fitz, Stephen and Kumaraguru, Ponnurangam and Karmakar, Kallol Krishna and Tupakula, Uday and Varadharajan, Vijay and Shoshitaishvili, Yan and Ba, Jimmy and Esvelt, Kevin M. and Wang, Alexandr and Hendrycks, Dan},
  booktitle = {Proceedings of the 41st International Conference on Machine Learning (ICML)},
  series    = {Proceedings of Machine Learning Research},
  volume    = {235},
  pages     = {28525--28550},
  publisher = {PMLR},
  year      = {2024},
  url       = {https://proceedings.mlr.press/v235/li24bc.html}
}

@inproceedings{fang2025alphaedit,
  title     = {AlphaEdit: Null-Space Constrained Knowledge Editing for Language Models},
  author    = {Fang, Junfeng and Jiang, Houcheng and Wang, Kun and Ma, Yunshan and Jie, Shi and Wang, Xiang and He, Xiangnan and Chua, Tat-Seng},
  booktitle = {The Thirteenth International Conference on Learning Representations (ICLR)},
  year      = {2025},
  note      = {Oral},
  url       = {https://arxiv.org/abs/2410.02355},
  doi       = {10.48550/arXiv.2410.02355}
}

@inproceedings{yao2024machine,
    title = "Machine Unlearning of Pre-trained Large Language Models",
    author = "Yao, Jin  and
      Chien, Eli  and
      Du, Minxin  and
      Niu, Xinyao  and
      Wang, Tianhao  and
      Cheng, Zezhou  and
      Yue, Xiang",
    editor = "Ku, Lun-Wei  and
      Martins, Andre  and
      Srikumar, Vivek",
    booktitle = "Proceedings of the 62nd Annual Meeting of the Association for Computational Linguistics (Volume 1: Long Papers)",
    month = aug,
    year = "2024",
    address = "Bangkok, Thailand",
    publisher = "Association for Computational Linguistics",
    url = "https://aclanthology.org/2024.acl-long.457/",
    doi = "10.18653/v1/2024.acl-long.457",
    pages = "8403--8419"
}

@inproceedings{NEURIPS2020_3fe78a8a,
 author = {Yu, Tianhe and Kumar, Saurabh and Gupta, Abhishek and Levine, Sergey and Hausman, Karol and Finn, Chelsea},
 booktitle = {Advances in Neural Information Processing Systems},
 pages = {5824--5836},
 publisher = {Curran Associates, Inc.},
 title = {Gradient Surgery for Multi-Task Learning},
 url = {https://proceedings.neurips.cc/paper_files/paper/2020/file/3fe78a8acf5fda99de95303940a2420c-Paper.pdf},
 volume = {33},
 year = {2020}
}

@article{touvron2023llama,
  title={LLaMA: Open and Efficient Foundation Language Models},
  author={Touvron, Hugo and Minervini, Marco and Chi, Emma and Figueroa, Arturo and others},
  journal={arXiv preprint arXiv:2302.13971},
  year={2023}
}

@article{ji2024reversing,
  title={Reversing the forget-retain objectives: An efficient llm unlearning framework from logit difference},
  author={Ji, Jiabao and Liu, Yujian and Zhang, Yang and Liu, Gaowen and Kompella, Ramana and Liu, Sijia and Chang, Shiyu},
  journal={Advances in Neural Information Processing Systems},
  volume={37},
  pages={12581--12611},
  year={2024}
}

@article{zhang2023review,
author = {Zhang, Haibo and Nakamura, Toru and Isohara, Takamasa and Sakurai, Kouichi},
title = {A Review on Machine Unlearning},
year = {2023},
issue_date = {May 2023},
publisher = {Springer-Verlag},
address = {Berlin, Heidelberg},
volume = {4},
number = {4},
url = {https://doi.org/10.1007/s42979-023-01767-4},
doi = {10.1007/s42979-023-01767-4},
journal = {SN Comput. Sci.},
month = apr,
numpages = {13}
}

@inproceedings{zhang-etal-2025-law,
    title = "The Law of Knowledge Overshadowing: Towards Understanding, Predicting and Preventing {LLM} Hallucination",
    author = "Zhang, Yuji  and
      Li, Sha  and
      Qian, Cheng  and
      Liu, Jiateng  and
      Yu, Pengfei  and
      Han, Chi  and
      Fung, Yi R.  and
      McKeown, Kathleen  and
      Zhai, ChengXiang  and
      Li, Manling  and
      Ji, Heng",
    editor = "Che, Wanxiang  and
      Nabende, Joyce  and
      Shutova, Ekaterina  and
      Pilehvar, Mohammad Taher",
    booktitle = "Findings of the Association for Computational Linguistics: ACL 2025",
    month = jul,
    year = "2025",
    address = "Vienna, Austria",
    publisher = "Association for Computational Linguistics",
    url = "https://aclanthology.org/2025.findings-acl.1199/",
    doi = "10.18653/v1/2025.findings-acl.1199",
    pages = "23340--23358",
    ISBN = "979-8-89176-256-5"
}

@inproceedings{liu-etal-2025-disentangling,
    title = "Disentangling Biased Knowledge from Reasoning in Large Language Models via Machine Unlearning",
    author = "Liu, Zheyuan  and
      Maharjan, Suraj  and
      Wu, Fanyou  and
      Parikh, Rahil  and
      Bayar, Belhassen  and
      Sengamedu, Srinivasan H.  and
      Jiang, Meng",
    editor = "Che, Wanxiang  and
      Nabende, Joyce  and
      Shutova, Ekaterina  and
      Pilehvar, Mohammad Taher",
    booktitle = "Proceedings of the 63rd Annual Meeting of the Association for Computational Linguistics (Volume 1: Long Papers)",
    month = jul,
    year = "2025",
    address = "Vienna, Austria",
    publisher = "Association for Computational Linguistics",
    url = "https://aclanthology.org/2025.acl-long.305/",
    doi = "10.18653/v1/2025.acl-long.305",
    pages = "6105--6123",
    ISBN = "979-8-89176-251-0"
}

@INPROCEEDINGS {thudi2022unrolling,
author = { Thudi, Anvith and Deza, Gabriel and Chandrasekaran, Varun and Papernot, Nicolas },
booktitle = { 2022 IEEE 7th European Symposium on Security and Privacy (EuroS\&P) },
title = {{ Unrolling SGD: Understanding Factors Influencing Machine Unlearning }},
year = {2022},
volume = {},
ISSN = {},
pages = {303-319},
doi = {10.1109/EuroSP53844.2022.00027},
url = {https://doi.ieeecomputersociety.org/10.1109/EuroSP53844.2022.00027},
publisher = {IEEE Computer Society},
address = {Los Alamitos, CA, USA},
month =Jun}

@inproceedings{wang2025rethinking,
    title={Rethinking {LLM} Unlearning Objectives: A Gradient Perspective and Go Beyond},
    author={Qizhou Wang and Jin Peng Zhou and Zhanke Zhou and Saebyeol Shin and Bo Han and Kilian Q Weinberger},
    booktitle={The Thirteenth International Conference on Learning Representations},
    year={2025},
    url={https://openreview.net/forum?id=huo8MqVH6t}
}

@article{rafailov2023dpo,
  title={Direct preference optimization: Your language model is secretly a reward model},
  author={Rafailov, Rafael and Sharma, Archit and Mitchell, Eric and Manning, Christopher D and Ermon, Stefano and Finn, Chelsea},
  journal={Advances in Neural Information Processing Systems},
  volume={36},
  pages={53728--53741},
  year={2023}
}

@article{liu2024large,
  title={Large language model unlearning via embedding-corrupted prompts},
  author={Liu, Chris and Wang, Yaxuan and Flanigan, Jeffrey and Liu, Yang},
  journal={Advances in Neural Information Processing Systems},
  volume={37},
  pages={118198--118266},
  year={2024}
}

@inproceedings{farajtabar2020orthogonal,
  title={Orthogonal gradient descent for continual learning},
  author={Farajtabar, Mehrdad and Azizan, Navid and Mott, Alex and Li, Ang},
  booktitle={International conference on artificial intelligence and statistics},
  pages={3762--3773},
  year={2020},
  organization={PMLR}
}

@article{kirkpatrick2017overcoming,
  title={Overcoming catastrophic forgetting in neural networks},
  author={Kirkpatrick, James and Pascanu, Razvan and Rabinowitz, Neil and Veness, Joel and Desjardins, Guillaume and Rusu, Andrei A and Milan, Kieran and Quan, John and Ramalho, Tiago and Grabska-Barwinska, Agnieszka and others},
  journal={Proceedings of the national academy of sciences},
  volume={114},
  number={13},
  pages={3521--3526},
  year={2017},
  publisher={National Academy of Sciences}
}

@inproceedings{wang2025llm,
    title={{LLM} Unlearning via Loss Adjustment with Only Forget Data},
    author={Yaxuan Wang and Jiaheng Wei and Chris Yuhao Liu and Jinlong Pang and Quan Liu and Ankit Shah and Yujia Bao and Yang Liu and Wei Wei},
    booktitle={The Thirteenth International Conference on Learning Representations},
    year={2025},
    url={https://openreview.net/forum?id=6ESRicalFE}
}

@article{ginart2019making,
  title={Making ai forget you: Data deletion in machine learning},
  author={Ginart, Antonio and Guan, Melody and Valiant, Gregory and Zou, James Y},
  journal={Advances in neural information processing systems},
  volume={32},
  year={2019}
}

@inproceedings{fan2024simplicity_simnpo,
    title={Simplicity Prevails: Rethinking Negative Preference Optimization for {LLM} Unlearning},
    author={Chongyu Fan and Jiancheng Liu and Licong Lin and Jinghan Jia and Ruiqi Zhang and Song Mei and Sijia Liu},
    booktitle={The Thirty-ninth Annual Conference on Neural Information Processing Systems},
    year={2025},
    url={https://openreview.net/forum?id=JbvSQm5h1l}
}

@inproceedings{dong-etal-2025-undial,
    title = "{UNDIAL}: Self-Distillation with Adjusted Logits for Robust Unlearning in Large Language Models",
    author = "Dong, Yijiang River  and
      Lin, Hongzhou  and
      Belkin, Mikhail  and
      Huerta, Ramon  and
      Vuli{\'c}, Ivan",
    booktitle = "Proceedings of the 2025 Conference of the Nations of the Americas Chapter of the Association for Computational Linguistics: Human Language Technologies (Volume 1: Long Papers)",
    month = apr,
    year = "2025",
    address = "Albuquerque, New Mexico",
    publisher = "Association for Computational Linguistics",
    url = "https://aclanthology.org/2025.naacl-long.444/",
    pages = "8827--8840",
    ISBN = "979-8-89176-189-6",
}

@article{yang2025u,
  title={{CE-U: Cross Entropy} Unlearning},
  author={Yang, Bo},
  journal={arXiv preprint arXiv:2503.01224},
  year={2025}
}

@inproceedings{tunstall2024zephyr,
    title={Zephyr: Direct Distillation of {LM} Alignment},
    author={Lewis Tunstall and Edward Emanuel Beeching and Nathan Lambert and Nazneen Rajani and Kashif Rasul and Younes Belkada and Shengyi Huang and Leandro Von Werra and Cl{\'e}mentine Fourrier and Nathan Habib and Nathan Sarrazin and Omar Sanseviero and Alexander M Rush and Thomas Wolf},
    booktitle={First Conference on Language Modeling},
    year={2024},
    url={https://openreview.net/forum?id=aKkAwZB6JV}
}

@inproceedings{wolf-etal-2020-transformers,
    title = "Transformers: State-of-the-Art Natural Language Processing",
    author = "Wolf, Thomas  and
      Debut, Lysandre  and
      Sanh, Victor  and
      Chaumond, Julien  and
      Delangue, Clement  and
      Moi, Anthony  and
      Cistac, Pierric  and
      Rault, Tim  and
      Louf, Remi  and
      Funtowicz, Morgan  and
      Davison, Joe  and
      Shleifer, Sam  and
      von Platen, Patrick  and
      Ma, Clara  and
      Jernite, Yacine  and
      Plu, Julien  and
      Xu, Canwen  and
      Le Scao, Teven  and
      Gugger, Sylvain  and
      Drame, Mariama  and
      Lhoest, Quentin  and
      Rush, Alexander",
    editor = "Liu, Qun  and
      Schlangen, David",
    booktitle = "Proceedings of the 2020 Conference on Empirical Methods in Natural Language Processing: System Demonstrations",
    month = oct,
    year = "2020",
    address = "Online",
    publisher = "Association for Computational Linguistics",
    url = "https://aclanthology.org/2020.emnlp-demos.6/",
    doi = "10.18653/v1/2020.emnlp-demos.6",
    pages = "38--45"
}

@inproceedings{shi2025muse,
  title={{MUSE}: Machine Unlearning Six-Way Evaluation for Language Models},
  author={Weijia Shi and Jaechan Lee and Yangsibo Huang and Sadhika Malladi and Jieyu Zhao and Ari Holtzman and Daogao Liu and Luke Zettlemoyer and Noah A. Smith and Chiyuan Zhang},
  booktitle={The Thirteenth International Conference on Learning Representations},
  year={2025},
  url={https://openreview.net/forum?id=TArmA033BU}
}

@article{NPO_zhang2024negative,
  title={Negative preference optimization: From catastrophic collapse to effective unlearning},
  author={Zhang, Ruiqi and Lin, Licong and Bai, Yu and Mei, Song},
  journal={First Conference on Language Modelling},
  year={2024},
  url={https://openreview.net/pdf?id=MXLBXjQkmb}
}

@inproceedings{bourtoule2021machine,
  title={Machine unlearning},
  author={Bourtoule, Lucas and Chandrasekaran, Varun and Choquette-Choo, Christopher A and Jia, Hengrui and Travers, Adelin and Zhang, Baiwu and Lie, David and Papernot, Nicolas},
  booktitle={2021 IEEE symposium on security and privacy (SP)},
  pages={141--159},
  year={2021},
  organization={IEEE}
}

@inproceedings{graves2021amnesiac,
  title={Amnesiac machine learning},
  author={Graves, Laura and Nagisetty, Vineel and Ganesh, Vijay},
  booktitle={Proceedings of the AAAI Conference on Artificial Intelligence},
  volume={35},
  pages={11516--11524},
  year={2021}
}

@article{maini2024tofu,
  title={{TOFU}: {A} task of fictitious unlearning for {LLMs}},
  author={Maini, Pratyush and Feng, Zhili and Schwarzschild, Avi and Lipton, Zachary C and Kolter, J Zico},
  journal={First Conference On Language Modeling},
  url={https://openreview.net/pdf?id=B41hNBoWLo},
  year={2024}
}

@article{grattafiori2024llama,
  title={The {Llama 3} Herd of Models},
  author={Grattafiori, Aaron and Dubey, Abhimanyu and Jauhri, Abhinav and Pandey, Abhinav and Kadian, Abhishek and Al-Dahle, Ahmad and Letman, Aiesha and Mathur, Akhil and Schelten, Alan and Vaughan, Alex and others},
  journal={arXiv preprint arXiv:2407.21783},
  year={2024},
  url={https://arxiv.org/abs/2407.21783}
}

@inproceedings{carlini2021extracting,
  title={Extracting training data from large language models},
  author={Carlini, Nicholas and Tramer, Florian and Wallace, Eric and Jagielski, Matthew and Herbert-Voss, Ariel and Lee, Katherine and Roberts, Adam and Brown, Tom and Song, Dawn and Erlingsson, Ulfar and others},
  booktitle={30th USENIX Security Symposium (USENIX Security 21)},
  pages={2633--2650},
  year={2021}
}

@inproceedings{amari2019fisher,
  title={Fisher information and natural gradient learning in random deep networks},
  author={Amari, Shun-ichi and Karakida, Ryo and Oizumi, Masafumi},
  booktitle={The 22nd International Conference on Artificial Intelligence and Statistics},
  pages={694--702},
  year={2019},
  organization={PMLR}
}

@article{duchi2011adaptive,
  title={Adaptive subgradient methods for online learning and stochastic optimization.},
  author={Duchi, John and Hazan, Elad and Singer, Yoram},
  journal={Journal of machine learning research},
  volume={12},
  number={7},
  year={2011}
}

@article{kingma2014adam,
  title={Adam: A Method for Stochastic Optimization},
  author={Diederik P. Kingma and Jimmy Ba},
  journal={CoRR},
  year={2014},
  volume={abs/1412.6980},
}

@inproceedings{martens2015optimizing,
  title={Optimizing neural networks with kronecker-factored approximate curvature},
  author={Martens, James and Grosse, Roger},
  booktitle={International conference on machine learning},
  pages={2408--2417},
  year={2015},
  organization={PMLR}
}

@inproceedings{hoang2024learn,
  title={Learn to unlearn for deep neural networks: Minimizing unlearning interference with gradient projection},
  author={Hoang, Tuan and Rana, Santu and Gupta, Sunil and Venkatesh, Svetha},
  booktitle={Proceedings of the IEEE/CVF Winter Conference on Applications of Computer Vision},
  pages={4819--4828},
  year={2024}
}

@article{cadet2024deep,
  title={Deep Unlearn: Benchmarking Machine Unlearning for Image Classification},
  author={Xavier F. Cadet and Anastasia Borovykh and M. Malekzadeh and Sara Ahmadi-Abhari and Hamed Haddadi},
  journal={2025 IEEE 10th European Symposium on Security and Privacy (EuroS\&P)},
  year={2024},
  pages={939-962}
}

@InProceedings{kim2024negmerge,
  title = 	 {{N}eg{M}erge: Sign-Consensual Weight Merging for Machine Unlearning},
  author =       {Kim, Hyo Seo and Han, Dongyoon and Choe, Junsuk},
  booktitle = 	 {Proceedings of the 42nd International Conference on Machine Learning},
  pages = 	 {30067--30085},
  year = 	 {2025},
  editor = 	 {Singh, Aarti and Fazel, Maryam and Hsu, Daniel and Lacoste-Julien, Simon and Berkenkamp, Felix and Maharaj, Tegan and Wagstaff, Kiri and Zhu, Jerry},
  volume = 	 {267},
  series = 	 {Proceedings of Machine Learning Research},
  month = 	 {13--19 Jul},
  publisher =    {PMLR},
  url = 	 {https://proceedings.mlr.press/v267/kim25f.html}
}

@ARTICLE{he2025TowardsNatural,
author={He, Zhengbao and Li, Tao and Cheng, Xinwen and Huang, Zhehao and Huang, Xiaolin},
journal={ IEEE Transactions on Pattern Analysis \& Machine Intelligence },
title={{ Towards Natural Machine Unlearning }},
year={2025},
volume={47},
number={12},
ISSN={1939-3539},
pages={11548-11560},
doi={10.1109/TPAMI.2025.3597350},
url = {https://doi.ieeecomputersociety.org/10.1109/TPAMI.2025.3597350},
publisher={IEEE Computer Society},
address={Los Alamitos, CA, USA},
month=dec}

@inproceedings{
    sendera2025semu,
    title={{SEMU}: Singular Value Decomposition for Efficient Machine Unlearning},
    author={Marcin Sendera and {\L}ukasz Struski and Kamil Ksi{k{a}}{\.z}ek and Kryspin Musiol and Jacek Tabor and Dawid Damian Rymarczyk},
    booktitle={Forty-second International Conference on Machine Learning},
    year={2025},
    url={https://openreview.net/forum?id=jnhkY0yCIW}
}

@inproceedings{chen-yang-2023-unlearn,
    title = "Unlearn What You Want to Forget: Efficient Unlearning for {LLM}s",
    author = "Chen, Jiaao  and
      Yang, Diyi",
    editor = "Bouamor, Houda  and
      Pino, Juan  and
      Bali, Kalika",
    booktitle = "Proceedings of the 2023 Conference on Empirical Methods in Natural Language Processing",
    month = dec,
    year = "2023",
    address = "Singapore",
    publisher = "Association for Computational Linguistics",
    url = "https://aclanthology.org/2023.emnlp-main.738/",
    doi = "10.18653/v1/2023.emnlp-main.738",
    pages = "12041--12052"
}

@inproceedings{chen2024unsc,
  title     = {Machine Unlearning via Null Space Calibration},
  author    = {Huiqiang Chen and Tianqing Zhu and Xin Yu and Wanlei Zhou},
  booktitle = {Proceedings of the 33rd International Joint Conference on Artificial Intelligence (IJCAI 2024)},
  year      = {2024},
  url       = {https://www.ijcai.org/proceedings/2024/0040.pdf}
}

@article{shamsian2025OrthoGrad,
  title={Go Beyond Your Means: Unlearning with Per-Sample Gradient Orthogonalization},
  author={Shamsian, Aviv and Shaar, Eitan and Navon, Aviv and Chechik, Gal and Fetaya, Ethan},
  journal={arXiv preprint arXiv:2503.02312},
  year={2025}
}

@article{mandal2025uno,
  title={UNO: Unlearning via Orthogonalization in Generative models},
  author={Mandal, Pinak and Gottwald, Georg A},
  journal={arXiv preprint arXiv:2506.04712},
  year={2025}
}

@inproceedings{feng2025FG-OrIU,
  title={FG-OrIU: Towards Better Forgetting via Feature-Gradient Orthogonality for Incremental Unlearning},
  author={Feng, Qian and Tu, JiaHang and Kang, Mintong and Zhao, Hanbin and Zhang, Chao and Qian, Hui},
  booktitle={Proceedings of the IEEE/CVF International Conference on Computer Vision},
  pages={1957--1967},
  year={2025}
}

@article{block2025machine,
  title={Machine Unlearning under Overparameterization},
  author={Block, Jacob L and Mokhtari, Aryan and Shakkottai, Sanjay},
  journal={arXiv preprint arXiv:2505.22601},
  year={2025}
}
